\documentclass{article} 
\usepackage{iclr2027_conference,times}

\usepackage{amsmath,amsfonts,bm}

\def\eqref#1{equation~\ref{#1}}

\def\1{\bm{1}}

\DeclareMathAlphabet{\mathsfit}{\encodingdefault}{\sfdefault}{m}{sl}
\SetMathAlphabet{\mathsfit}{bold}{\encodingdefault}{\sfdefault}{bx}{n}

\usepackage{booktabs,multirow,graphicx,xcolor,colortbl,textcomp}

\definecolor{gdloraBlue}{RGB}{221,235,247}

\providecommand{\gluescore}[2]{}
\renewcommand{\gluescore}[2]{%
    #1\textsubscript{\textpm\,#2}%
}

\providecommand{\glueavg}[1]{}
\renewcommand{\glueavg}[1]{#1}

\usepackage{hyperref}
\usepackage{url}
\usepackage{graphicx}
\usepackage{amsmath}
\usepackage{amsthm}
\usepackage{booktabs}
\usepackage{graphicx}
\usepackage{algorithm}
\usepackage{algpseudocode}
\newtheorem{theorem}{Theorem}
\newtheorem{proposition}{Proposition}
\usepackage{subcaption}
\title{Beyond Low-Rank Parameterization:\\ Narrowing the Gap Between LoRA and Full Fine-Tuning via Gradient Decomposition}

\author{%
  {\bf Yihao Ouyang}$^{1}$, 
  {\bf Shiwei Li}$^{1}$, 
  {\bf Haozhao Wang}$^{1}$, 
  {\bf Xiandi Luo}$^{1}$, 
  {\bf Zhuoqi Hu}$^{1}$, \\ 
  {\bf Jinglun Yu}$^{2}$, 
  {\bf Yichen Li}$^{1}$, 
  {\bf Ruixuan Li}$^{1}$ \\
  $^1$ Huazhong University of Science and Technology, Wuhan, China \\
  $^2$ Hebei University of Technology , Tianjin, China
}

\iclrfinalcopy
\begin{document}

\maketitle
\lhead{Preprint}

\begin{abstract}

Low-Rank Adaptation (LoRA) is a widely used approach to
parameter-efficient fine-tuning (PEFT), yet a performance gap can remain
relative to full fine-tuning (FFT). Many LoRA variants improve the
initialization or optimization of low-rank factors. At each training step, however, their first-order weight-space directions are constrained
by the current parameterization. We characterize the corresponding
\textbf{LoRA-accessible gradient space} and show that it coincides with the
tangent space induced by the current LoRA parameterization. This
characterization yields an orthogonal decomposition of the full weight
gradient at the current model parameters. We term the component orthogonal
to this space the \textbf{normal gradient}.
Based on this decomposition, we propose
\textbf{GDLoRA (Gradient-Decomposed Low-Rank Adaptation)}. GDLoRA reconstructs
the full weight gradient from forward activations and backward signals,
extracts its normal component, and directly updates the base weights with
this component, while retaining standard AdamW optimization for the LoRA
factors. GDLoRA incorporates complementary normal gradients without
increasing standard LoRA's optimizer-state memory budget under matched
adapter and optimizer configurations. Experiments on natural language
understanding, mathematical reasoning, commonsense reasoning, and image
classification show that GDLoRA consistently improves over LoRA and
narrows the performance gap to FFT.
The code is available at \url{https://anonymous.4open.science/r/GDLoRA}.

\end{abstract}

\section{Introduction}

Large foundation models achieve strong performance across diverse
applications~\citep{devlin2019bert,brown2020language,chowdhery2023palm},
but adapting them through full fine-tuning (FFT) incurs substantial
computation and memory costs for gradients and optimizer
states~\citep{rajbhandari2020zero}.
Parameter-efficient fine-tuning (PEFT) reduces these costs by optimizing
a small subset of parameters while keeping most pretrained weights
frozen~\citep{houlsby2019parameter,li2021prefix,lester2021power}.
Among these methods, Low-Rank Adaptation (LoRA)~\citep{hu2021lora}
introduces trainable low-rank factors into pretrained layers, reducing
optimizer-state memory while maintaining competitive downstream
performance.

Despite these advantages, a performance gap can remain between LoRA
and FFT. A related line of research seeks to enable full-rank
training with reduced optimizer-state memory.
Fira~\citep{Chen2024FiraCW} pursues this goal by combining adaptive
optimization of low-rank projected gradients with scaled residual
updates. Its projection bases are obtained through SVD and reused
across multiple optimization steps to amortize decomposition costs.
However, as the leading gradient subspace drifts during training,
reused bases may become less aligned with important directions of
subsequent gradients, potentially limiting the effectiveness of
adaptive optimization, while more frequent basis refreshes incur
additional computational costs. A detailed analysis is provided in
Appendix~\ref{app:projection_refresh}.
Another line of research improves LoRA optimization through
adapter initialization and gradient alignment.
LoRA-GA~\citep{wang2024lora} uses full-gradient information to
initialize the factors so that their initial induced weight-space
gradient better approximates that of full fine-tuning.
LoRA-Pro~\citep{wang2025lora} studies the equivalent weight-space
gradient induced by the adapter gradients and adjusts them to
improve its approximation to the full weight gradient.

Although these adapter-based methods do not require periodic
full-gradient SVD refreshes, their first-order weight-space updates
remain restricted to directions accessible through the current
LoRA parameterization.
This motivates us to investigate gradient information beyond
these accessible directions.
Although the accessible space evolves as the factors change,
gradient components orthogonal to it cannot be expressed through
first-order factor perturbations at the current iterate.
We therefore ask: \textbf{\textit{Can we explicitly recover and
exploit this complementary gradient information while retaining
standard LoRA's optimizer-state memory budget?}}

To address this question, we characterize the
\textbf{LoRA-accessible gradient space}, consisting of all equivalent
weight-space gradients induced by the Euclidean factor gradients, and show
that it coincides with the tangent space induced by the current LoRA
parameterization. This characterization yields an orthogonal decomposition
of the full weight gradient evaluated at the current model parameters. We
call the component orthogonal to the accessible space the
\textbf{normal gradient}. When nonzero, its negative provides a local descent
direction for the current batch loss with the adapters and other weights
held fixed. It is also orthogonal to every first-order weight-space direction
available to the current adapters. These properties motivate using the normal
component to supplement standard LoRA optimization.

Based on this decomposition, we propose
\textbf{GDLoRA (Gradient-Decomposed Low-Rank Adaptation)}. GDLoRA reconstructs
the full weight gradient of LoRA-injected layers from forward activations and
backward signals, normalizes the recovered gradient by the number of input
positions, and extracts its normal component using the
current LoRA tangent space. This additional token-count normalization applies
only to the normal branch and preserves the gradient direction. The adapter parameters follow standard AdamW optimization,
while the token-normalized normal gradient directly updates the base weights initialized from
the pretrained model. Both branches use quantities computed from the same
pre-update parameters. The normal branch introduces no additional momentum
or variance buffers for the base weights. In this way, GDLoRA approaches or
matches the performance of full fine-tuning in the evaluated settings,
without introducing additional optimizer-state memory overhead relative to
standard LoRA under matched adapter and optimizer configurations.

Our contributions are summarized as follows:

\begin{itemize}

\item We characterize the \textbf{LoRA-accessible gradient space} through
its equivalence to the tangent space induced by the current parameterization,
and identify the \textbf{normal gradient} as a complementary component with
a local descent interpretation.

\item We propose \textbf{GDLoRA}, which retains standard AdamW optimization
of the LoRA factors and applies the token-normalized normal gradient
directly to the base weights, without introducing additional
base-weight momentum or variance buffers.

\item Across natural language understanding, mathematical and commonsense
reasoning, and image classification, GDLoRA improves average scores
by 2-4\% relative to LoRA, approaching FFT without additional
optimizer-state memory under matched configurations.

\end{itemize}

\section{Related work}

\paragraph{Parameter-Efficient Fine-Tuning.}

Given the increasing scale of foundation models, parameter-efficient fine-tuning (PEFT) methods have been extensively studied to reduce the computational and memory costs of model adaptation \citep{houlsby2019parameter,li2021prefix,lester2021power,hu2021lora}. Instead of updating all parameters during fine-tuning, PEFT methods optimize only a small subset of additional parameters while keeping the pretrained model largely frozen.
Existing PEFT methods can be broadly categorized into several groups. One line of research explores adapter-based tuning \citep{houlsby2019parameter,pfeiffer2020adapterhub}, which introduces lightweight trainable modules into pretrained networks and updates only these additional parameters during adaptation. Another line focuses on prompt-based tuning \citep{li2021prefix,lester2021power}, which adapts pretrained models by optimizing learnable prompts or virtual tokens without modifying the model parameters. In a related direction, Low-Rank Adaptation (LoRA) \citep{hu2021lora} introduces trainable
low-rank matrices into pretrained layers while keeping the base weights
frozen, reducing the number of optimized parameters and the associated
optimizer-state memory. In this work, we focus on improving the optimization capability of LoRA-based adaptation.

\paragraph{Low-Rank Adaptation.}

Low-Rank Adaptation (LoRA)~\citep{hu2021lora} adapts pretrained models
through trainable low-rank updates, balancing efficiency and performance.
Subsequent variants improve initialization, optimization, and adaptation
capability. PiSSA~\citep{meng2024pissa} initializes LoRA using principal
singular components of pretrained weights.
LoRA-GA~\citep{wang2024lora} aligns initial LoRA-induced gradients with
full fine-tuning gradients, while LoRA-One~\citep{zhang2025lora} studies
the relationship between LoRA adapters and full-gradient subspaces.
LoRA+~\citep{hayou2024lora+} assigns different learning rates to the two
factors, and rsLoRA~\citep{kalajdzievski2023rank} uses rank-dependent
scaling to stabilize training.
AdaLoRA~\citep{zhang2023adalora} dynamically allocates rank budgets
across layers. DoRA~\citep{liu2024dora} decomposes weights into magnitude
and direction, while LoRA-Pro~\citep{wang2025lora} adjusts adapter
gradients to better approximate full weight gradients.
GaLore~\citep{Zhao2024GaLoreML} reduces optimizer-state memory via
low-rank gradient projection; Fira~\citep{Chen2024FiraCW} adds scaled
residual updates for full-rank training. Both reuse SVD bases, trading
subspace alignment against refresh costs. GDLoRA instead retains standard
LoRA optimization and derives projection bases from the current adapters
each step, updating base weights along the normal gradient without
full-gradient SVD.

\section{Method}

\subsection{Full Gradient Recovery}
\label{sec:full_gradient_recovery}

GDLoRA supplements standard LoRA optimization with gradient
information orthogonal to the first-order directions accessible
through the current LoRA parameterization.
We first describe full weight-gradient recovery and the
token-count normalization used by the normal branch.

For a LoRA-injected linear layer, let
$W_{\mathrm{pt}}\in\mathbb{R}^{d_o\times d_i}$ denote the
pretrained weight, and let $W_{\mathrm{base},t}$ denote the
base weight at iteration $t$, initialized as
$W_{\mathrm{base},0}=W_{\mathrm{pt}}$.
The effective weight is
\begin{equation}
W'_t=W_{\mathrm{base},t}+sB_tA_t,
\label{eq:gdlora_effective_weight}
\end{equation}
where $B_t\in\mathbb{R}^{d_o\times r}$ and
$A_t\in\mathbb{R}^{r\times d_i}$ are the LoRA factors,
and $s\neq0$ is a fixed scaling factor (e.g., $s=\alpha/r$).

For input activations $X\in\mathbb{R}^{N\times d_i}$,
the batch and token dimensions are flattened into $N$ rows.
The layer output is
\begin{equation}
Y=X(W')^T\in\mathbb{R}^{N\times d_o}.
\end{equation}
Let
\begin{equation}
\delta=\frac{\partial\mathcal{L}}{\partial Y}
\in\mathbb{R}^{N\times d_o}
\end{equation}
denote the backward signal.
By the chain rule, the full weight gradient is
\begin{equation}
G_{\mathrm{raw}}=\frac{\partial\mathcal{L}}{\partial W'}
=\delta^TX\in\mathbb{R}^{d_o\times d_i}.
\label{eq:gdlora_gradient_recovery}
\end{equation}
Here, $G_{\mathrm{raw}}$ is the unconstrained weight gradient at the current
model parameters, and any loss scaling applied before
backpropagation is already reflected in $\delta$.

For the normal branch, GDLoRA further normalizes each
recovered gradient by the number of input positions:
\begin{equation}
G=\frac{G_{\mathrm{raw}}}{N}
=\frac{\delta^TX}{N}.
\label{eq:gdlora_token_normalization}
\end{equation}
In preliminary runs, we observed unstable training with unnormalized
normal-gradient updates, and selecting a suitable learning rate was
difficult within the settings explored. This motivated scaling the normal
gradient by $1/N$, where $N$ counts the input positions.
This preserves its direction and controls the effective step size,
equivalently rescaling the learning rate for fixed $N$.
Adapter gradients remain unchanged. Appendix~\ref{app:normal_update_stability}
provides a magnitude bound and local descent analysis.
For the following local analysis, let $\overline{\mathcal{L}}=\mathcal{L}/N$,
with $N$ fixed for the current batch. Adapter optimization continues to
use $\mathcal{L}$.

Gradient recovery adds a matrix multiplication relative to
standard LoRA but requires no momentum or variance buffers
for the base weights.
Full-weight gradient computation is also required by
Fira~\citep{Chen2024FiraCW}.
GDLoRA constructs its projection bases from thin adapter
factors, avoiding SVD of the full gradient.
The computational comparison is provided in
Appendix~\ref{app:gdlora_computational_cost}.

\begin{figure*}[t]
    \centering
    \includegraphics[width=\textwidth]{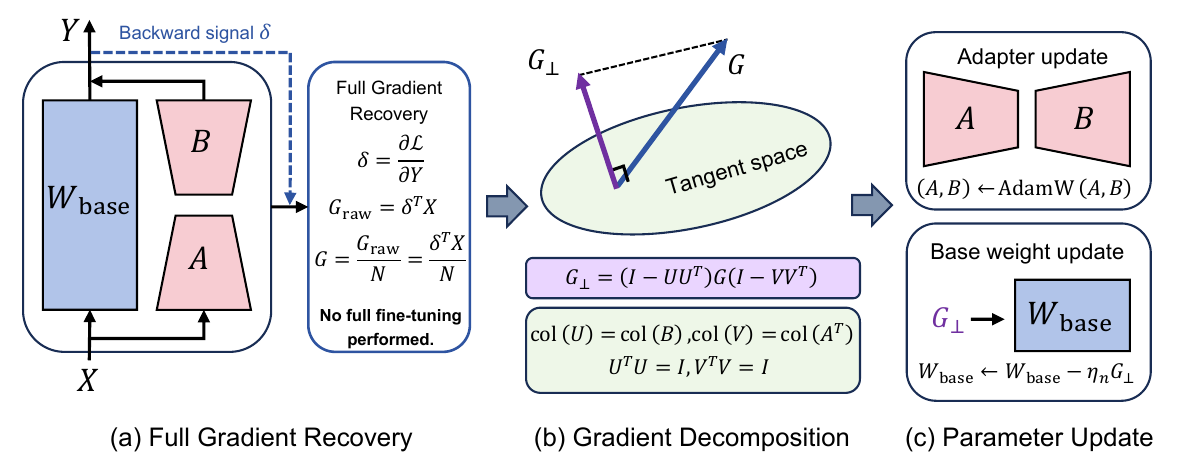}
    \caption{
    \textbf{Overview of GDLoRA.}
    \textbf{(a)} GDLoRA reconstructs the full weight gradient
    at the current model parameters from forward activations and backward
    signals, and obtains $G$ as defined in Section~\ref{sec:full_gradient_recovery}.
    \textbf{(b)} The gradient is decomposed with respect to the current
    parameterization-induced LoRA tangent space, and its orthogonal
    component $G_{\perp}$ is the \emph{normal gradient}.
    \textbf{(c)} The adapters follow standard AdamW optimization, while
    $G_{\perp}$ updates $W_{\mathrm{base}}$ without additional base-weight
    momentum or variance buffers. The two branches are complementary
    at first order at the current iterate.
    }
    \label{fig:GDLoRA_overview}
\end{figure*}

\subsection{Gradient Decomposition}
\label{sec:gradient_decomposition}

Our goal is to add a gradient component that is orthogonal to all
first-order weight-space directions available to the current adapters,
while retaining their standard optimization rule. This provides a precise
notion of complementary information; overlap with adapter directions is
not, by itself, assumed to be harmful.

\paragraph{First-order accessible directions.}
For the parameterization $F_s(A,B)=sBA$, its differential maps factor
perturbations to weight-space directions as
\begin{equation}
J_s(H_A,H_B)
=DF_s(A,B)[H_A,H_B]
=s(BH_A+H_BA).
\label{eq:gdlora_jacobian}
\end{equation}
We define the parameterization-induced tangent space by
\begin{equation}
\mathcal{T}_{\mathrm{LoRA}}
=\operatorname{Range}(J_s)
=\left\{s(BH_A+H_BA)
\;\middle|\;
H_A\in\mathbb{R}^{r\times d_i},\,
H_B\in\mathbb{R}^{d_o\times r}\right\}.
\end{equation}
Throughout, ``accessible'' refers to these first-order directions at the
current $(A,B)$. This definition remains valid for rank-deficient factors.
When both factors have rank $r$, it coincides with the usual tangent
space of the rank-$r$ matrix manifold at $sBA$, as used in geometric
approaches to LoRA optimization.

By the chain rule, the factor gradients are
\begin{equation}
\nabla_A\mathcal{L}=sB^{T}G_{\mathrm{raw}},
\qquad
\nabla_B\mathcal{L}=sG_{\mathrm{raw}}A^{T}.
\label{eq:gdlora_factor_gradients}
\end{equation}
Mapping these Euclidean gradients through $J_s$ and dividing by $N$ gives
the first-order weight-space gradient direction
\begin{align}
G_{\mathrm{LoRA}}(G)
&=\frac{s}{N}\bigl(B\nabla_A\mathcal{L}+\nabla_B\mathcal{L}A\bigr)\nonumber\\
&=s^2\bigl(BB^{T}G+GA^{T}A\bigr).
\label{eq:gdlora_equivalent_gradient}
\end{align}
This is the Euclidean-gradient-induced direction underlying the
equivalent-gradient view of LoRA; it is not the
actual AdamW update. Define the corresponding accessible gradient space as
\begin{equation}
\mathcal{G}_{\mathrm{LoRA}}
=\operatorname{Range}(\Phi_s),
\qquad
\Phi_s(H)=s^2\bigl(BB^{T}H+HA^{T}A\bigr).
\end{equation}

Let $U\in\mathbb{R}^{d_o\times r_B}$ and
$V\in\mathbb{R}^{d_i\times r_A}$ be orthonormal bases satisfying
\begin{equation}
\begin{gathered}
U^{T}U=I,\qquad V^{T}V=I,\\
\operatorname{col}(U)=\operatorname{col}(B),\qquad
\operatorname{col}(V)=\operatorname{col}(A^{T}),
\end{gathered}
\end{equation}
where $r_B=\operatorname{rank}(B)$ and $r_A=\operatorname{rank}(A)$.
A zero-rank factor has an empty basis and a zero associated projector.

\begin{theorem}[LoRA-Accessible Gradient Space Characterization]
\label{thm:lora_gradient_space}
For any $A,B$ and fixed $s\neq 0$,
\begin{equation}
\mathcal{G}_{\mathrm{LoRA}}
=\mathcal{T}_{\mathrm{LoRA}}
=\left\{UK+LV^{T}
\;\middle|\;
K\in\mathbb{R}^{r_B\times d_i},\,
L\in\mathbb{R}^{d_o\times r_A}\right\}.
\end{equation}
\end{theorem}

The proof is provided in Appendix~\ref{app:proof_gradient_space}.
This characterization identifies the orthogonal complement of the
first-order LoRA-accessible space used in the following decomposition.

\paragraph{Normal-gradient projection.}
The operator $\Phi_s$ is generally not an orthogonal projector.
Consequently, subtracting $G_{\mathrm{LoRA}}(G)$ from $G$ does not
generally produce a residual orthogonal to
$\mathcal{G}_{\mathrm{LoRA}}$.
The matrices $BB^{T}$ and $A^{T}A$ are also generally not orthogonal
projectors and therefore cannot directly serve as the column-
and row-space projectors required for this decomposition.

By Theorem~\ref{thm:lora_gradient_space}, the LoRA-accessible
gradient space is exactly the tangent space induced by the
current LoRA parameterization.
Thus, projecting $G$ onto $\mathcal{G}_{\mathrm{LoRA}}$ and its
orthogonal complement is equivalent to projecting it onto
$\mathcal{T}_{\mathrm{LoRA}}$ and
$\mathcal{T}_{\mathrm{LoRA}}^{\perp}$, respectively.
Using the orthonormal bases $U$ and $V$, the corresponding
column- and row-space projectors are $UU^{T}$ and $VV^{T}$,
yielding the orthogonal decomposition
\begin{align}
G &= G_{\parallel}+G_{\perp},\\
G_{\parallel}
&=UU^{T}G+GVV^{T}-UU^{T}GVV^{T},\\
G_{\perp}
&=(I-UU^{T})G(I-VV^{T}).
\label{eq:gdlora_normal_gradient}
\end{align}
These are the standard two-sided tangent and normal projectors in the
full-rank-factor case; the basis formulation
also covers rank-deficient factors under our first-order-space definition.
In particular,
\begin{equation}
G_{\parallel}\in\mathcal{T}_{\mathrm{LoRA}},\qquad
G_{\perp}\in\mathcal{T}_{\mathrm{LoRA}}^{\perp},\qquad
\langle G_{\perp},Z\rangle_F=0
\quad\forall Z\in\mathcal{T}_{\mathrm{LoRA}}.
\end{equation}
We call $G_{\perp}$ the \textbf{normal gradient}. A direct verification of
the projection is provided in Appendix~\ref{app:normal_projection}.

\paragraph{Why use the normal component?}
The normal update has a simple local optimization interpretation.
Let $D$ denote an additive update to $W_{\mathrm{base}}$.
With the adapters and all other weights fixed, the first-order
change in $\overline{\mathcal{L}}$ is $\langle G,D\rangle_F$.
We seek an update orthogonal to the current first-order adapter
space while penalizing its magnitude:
\begin{equation}
\underset{D\in\mathcal{T}_{\mathrm{LoRA}}^{\perp}}{\operatorname{arg\,min}}
\left\{\langle G,D\rangle_F
+\frac{1}{2\eta_n}\|D\|_F^2\right\}
=-\eta_nG_{\perp},\qquad \eta_n>0.
\label{eq:gdlora_normal_step_objective}
\end{equation}
To see this, orthogonality gives
$\langle G_{\parallel},D\rangle_F=0$ and hence
$\langle G,D\rangle_F=\langle G_{\perp},D\rangle_F$.
Completing the square yields
\begin{equation*}
\langle G,D\rangle_F+\frac{1}{2\eta_n}\|D\|_F^2
=
\frac{1}{2\eta_n}\|D+\eta_nG_{\perp}\|_F^2
-\frac{\eta_n}{2}\|G_{\perp}\|_F^2.
\end{equation*}
The last term is independent of $D$, and the squared term
is minimized at $D=-\eta_nG_{\perp}$.

Furthermore, treating $G_{\perp}$ as a fixed direction computed
at the current parameters, differentiability gives
\begin{equation}
\left.\frac{d}{d\eta}
\overline{\mathcal{L}}(W_{\mathrm{base}}-\eta G_{\perp},A,B)
\right|_{\eta=0}
=-\langle G,G_{\perp}\rangle_F
=-\|G_{\perp}\|_F^2.
\label{eq:gdlora_normal_descent}
\end{equation}
Thus, whenever $G_{\perp}\neq 0$, a sufficiently small step
along $-G_{\perp}$ decreases the current batch loss while
remaining orthogonal to the current first-order adapter space.
This local property motivates the decomposition.

\subsection{Parameter Update}
\label{sec:parameter_update}

At iteration $t$, GDLoRA computes $G_t$ and the projection bases from the
same pre-update parameters. The adapters follow standard AdamW using the
factor gradients in Eq.~(\ref{eq:gdlora_factor_gradients}):
\begin{equation}
(A_{t+1},B_{t+1})
=\operatorname{AdamW}(A_t,B_t;\eta_{a,t}),
\label{eq:gdlora_adapter_update}
\end{equation}
where $\eta_{a,t}$ is the adapter learning rate and the optimizer state is
implicit. The normal branch updates the base weight directly:
\begin{equation}
W_{\mathrm{base},t+1}
=W_{\mathrm{base},t}-\eta_{n,t}G_{\perp,t},
\label{eq:gdlora_base_update}
\end{equation}
where $\eta_{n,t}\geq 0$ is the normal-gradient learning rate.
Both updates are applied after backward propagation. Setting $\eta_{n,t}=0$ throughout recovers
standard LoRA under the same training configuration.

The normal-gradient branch introduces no additional
momentum or variance buffers for the base weights, preserving standard
LoRA's optimizer-state budget under matched adapter and optimizer
configurations. At the current iterate, the normal update is orthogonal
to all first-order weight-space directions induced by the LoRA factors.
The finite-step relationship between the two branches is discussed in
Appendix~\ref{app:gdlora_local_analysis}.

\section{Experiments}

We evaluate GDLoRA on natural language understanding, mathematical reasoning,
commonsense reasoning, and image classification, using RoBERTa, Llama-3-8B,
Gemma-7B, and ViT models, respectively. These experiments span GLUE,
mathematical and commonsense reasoning benchmarks, and eight image
classification datasets. Dataset details, adapter configurations, and
hyperparameters are provided in Appendix~\ref{app:experimental_details}.

\begin{table}[t]
    \centering
    \caption{
        Results on six GLUE tasks with RoBERTa-Base and RoBERTa-Large.
        All LoRA-based methods use rank 8.
        For each model, the best and second-best scores among all
        methods, including FFT, are shown in \textbf{bold} and
        \underline{underlined}, respectively.
    }
    \label{tab:glue_results}

    \footnotesize
    \setlength{\tabcolsep}{3.5pt}
    \renewcommand{\arraystretch}{1.15}

    \resizebox{1.0\linewidth}{!}{%
    \begin{tabular}{clccccccc}
        \toprule
        Model & Method & SST-2 & MRPC & CoLA & QNLI & RTE & STS-B & Avg. \\
        \midrule

        & FFT
        & \gluescore{\textbf{94.42}}{0.18}
        & \gluescore{\textbf{89.24}}{0.61}
        & \gluescore{\underline{64.53}}{1.08}
        & \gluescore{\underline{92.81}}{0.22}
        & \gluescore{\underline{78.96}}{0.84}
        & \gluescore{\textbf{91.24}}{0.19}
        & \glueavg{\underline{85.20}} \\
        \cmidrule(l){2-9}

        & LoRA
        & \gluescore{93.04}{0.27}
        & \gluescore{87.12}{0.74}
        & \gluescore{62.89}{1.31}
        & \gluescore{92.31}{0.33}
        & \gluescore{75.81}{1.12}
        & \gluescore{89.49}{0.24}
        & \glueavg{83.44} \\

        & PiSSA
        & \gluescore{93.15}{0.41}
        & \gluescore{87.09}{0.58}
        & \gluescore{60.98}{0.73}
        & \gluescore{92.02}{0.68}
        & \gluescore{77.98}{0.97}
        & \gluescore{90.32}{0.17}
        & \glueavg{83.59} \\

        & DoRA
        & \gluescore{93.03}{0.34}
        & \gluescore{87.25}{0.81}
        & \gluescore{62.77}{1.46}
        & \gluescore{92.38}{0.19}
        & \gluescore{76.53}{1.28}
        & \gluescore{89.54}{0.31}
        & \glueavg{83.58} \\

        & rsLoRA
        & \gluescore{93.06}{0.16}
        & \gluescore{87.73}{0.49}
        & \gluescore{61.62}{1.12}
        & \gluescore{92.77}{0.37}
        & \gluescore{77.62}{0.88}
        & \gluescore{89.87}{0.22}
        & \glueavg{83.78} \\

        & Fira
        & \gluescore{92.51}{0.29}
        & \gluescore{86.92}{0.67}
        & \gluescore{60.64}{1.24}
        & \gluescore{92.62}{0.26}
        & \gluescore{77.34}{1.05}
        & \gluescore{89.99}{0.18}
        & \glueavg{83.34} \\

        & LoRA-GA
        & \gluescore{94.15}{0.23}
        & \gluescore{87.01}{0.72}
        & \gluescore{62.16}{1.37}
        & \gluescore{92.26}{0.31}
        & \gluescore{76.47}{0.91}
        & \gluescore{89.98}{0.26}
        & \glueavg{83.67} \\

        & LoRA-Pro
        & \gluescore{\underline{94.38}}{0.31}
        & \gluescore{87.99}{0.86}
        & \gluescore{60.79}{1.58}
        & \gluescore{91.93}{0.42}
        & \gluescore{73.29}{1.34}
        & \gluescore{89.75}{0.29}
        & \glueavg{83.02} \\

        \rowcolor{gdloraBlue}
        \cellcolor{white}
        \multirow{-9}{*}{%
            \rotatebox[origin=c]{90}{RoBERTa-Base}%
        }
        & \textbf{GDLoRA}
        & \gluescore{94.19}{0.14}
        & \gluescore{\underline{88.97}}{0.64}
        & \gluescore{\textbf{65.49}}{1.05}
        & \gluescore{\textbf{92.93}}{0.18}
        & \gluescore{\textbf{79.42}}{0.96}
        & \gluescore{\underline{90.57}}{0.25}
        & \glueavg{\textbf{85.26}} \\

        \midrule

        & FFT
        & \gluescore{96.24}{0.17}
        & \gluescore{\textbf{91.52}}{0.95}
        & \gluescore{\underline{68.45}}{1.21}
        & \gluescore{\textbf{94.97}}{0.24}
        & \gluescore{\underline{87.93}}{0.93}
        & \gluescore{\textbf{92.67}}{0.32}
        & \glueavg{\underline{88.63}} \\
        \cmidrule(l){2-9}

        & LoRA
        & \gluescore{95.72}{0.41}
        & \gluescore{89.76}{0.92}
        & \gluescore{65.74}{1.63}
        & \gluescore{94.03}{0.36}
        & \gluescore{83.15}{1.27}
        & \gluescore{91.46}{0.47}
        & \glueavg{86.64} \\

        & PiSSA
        & \gluescore{\underline{96.33}}{0.22}
        & \gluescore{90.69}{0.63}
        & \gluescore{66.44}{0.97}
        & \gluescore{94.44}{0.29}
        & \gluescore{86.28}{0.88}
        & \gluescore{91.95}{0.18}
        & \glueavg{87.69} \\

        & DoRA
        & \gluescore{95.89}{0.35}
        & \gluescore{89.95}{0.84}
        & \gluescore{65.92}{1.42}
        & \gluescore{94.27}{0.21}
        & \gluescore{85.21}{1.14}
        & \gluescore{92.12}{0.39}
        & \glueavg{87.23} \\

        & rsLoRA
        & \gluescore{95.87}{0.26}
        & \gluescore{89.28}{0.71}
        & \gluescore{66.56}{1.18}
        & \gluescore{94.42}{0.33}
        & \gluescore{85.64}{1.02}
        & \gluescore{92.01}{0.27}
        & \glueavg{87.30} \\

        & Fira
        & \gluescore{95.76}{0.19}
        & \gluescore{89.05}{0.58}
        & \gluescore{65.78}{1.05}
        & \gluescore{94.51}{0.27}
        & \gluescore{85.28}{0.91}
        & \gluescore{91.95}{0.22}
        & \glueavg{87.06} \\

        & LoRA-GA
        & \gluescore{96.08}{0.28}
        & \gluescore{90.69}{0.76}
        & \gluescore{65.78}{1.36}
        & \gluescore{94.25}{0.31}
        & \gluescore{85.56}{1.09}
        & \gluescore{92.29}{0.34}
        & \glueavg{87.44} \\

        & LoRA-Pro
        & \gluescore{96.22}{0.32}
        & \gluescore{89.71}{0.88}
        & \gluescore{66.77}{1.49}
        & \gluescore{93.85}{0.43}
        & \gluescore{85.57}{1.21}
        & \gluescore{91.79}{0.41}
        & \glueavg{87.32} \\

        \rowcolor{gdloraBlue}
        \cellcolor{white}
        \multirow{-9}{*}{%
            \rotatebox[origin=c]{90}{RoBERTa-Large}%
        }
        & \textbf{GDLoRA}
        & \gluescore{\textbf{96.56}}{0.25}
        & \gluescore{\underline{90.93}}{0.49}
        & \gluescore{\textbf{68.73}}{1.08}
        & \gluescore{\underline{94.76}}{0.23}
        & \gluescore{\textbf{88.95}}{0.89}
        & \gluescore{\underline{92.31}}{0.36}
        & \glueavg{\textbf{88.71}} \\

        \bottomrule
    \end{tabular}%
    }
\end{table}

\subsection{Natural Language Understanding}
\label{sec:results_glue}

As shown in Table~\ref{tab:glue_results}, GDLoRA achieves the highest
average scores among all evaluated methods on RoBERTa-Base and
RoBERTa-Large, reaching 85.26 and 88.71, respectively.
Compared with standard LoRA at rank 8, the average improvements
are 1.82 and 2.07 points.
GDLoRA also surpasses the strongest non-FFT baseline for each
backbone, namely rsLoRA on RoBERTa-Base and PiSSA on
RoBERTa-Large, by 1.48 and 1.02 points, respectively.
Against Fira, the corresponding improvements are 1.92 and
1.65 points.
This advantage extends to individual tasks: GDLoRA outperforms
both standard LoRA and Fira in all 12 model--task combinations
and achieves the highest non-FFT score in 11 of them.
The largest gains over standard LoRA occur on RTE, reaching
3.61 and 5.80 points.
Its average scores closely match FFT's 85.20 and 88.63,
while the normal branch requires no additional base-weight
momentum or variance buffers.


\begin{table}[t]
    \centering
    \caption{
        Results on six mathematical reasoning tasks with Llama-3-8B.
        All LoRA-based methods use rank 8.
        The best and second-best scores among all methods, including
        FFT, are shown in \textbf{bold} and \underline{underlined},
        respectively.
    }
    \label{tab:math_results}

    \footnotesize
    \setlength{\tabcolsep}{3.5pt}
    \renewcommand{\arraystretch}{1.15}

    \resizebox{1.0\linewidth}{!}{%
    \begin{tabular}{lccccccc}
        \toprule
        Method & AddSub & MultiArith & SingleEq & SVAMP & GSM8K & AQuA & Avg. \\
        \midrule

        FFT
        & \gluescore{\textbf{86.78}}{1.37}
        & \gluescore{\underline{90.21}}{1.83}
        & \gluescore{\textbf{93.21}}{0.84}
        & \gluescore{\textbf{70.87}}{1.26}
        & \gluescore{\textbf{56.72}}{0.75}
        & \gluescore{\underline{25.96}}{1.54}
        & \glueavg{\textbf{70.63}} \\
        \midrule

        LoRA
        & \gluescore{81.52}{1.54}
        & \gluescore{87.28}{1.11}
        & \gluescore{91.54}{0.69}
        & \gluescore{66.93}{1.59}
        & \gluescore{55.34}{0.66}
        & \gluescore{22.44}{0.60}
        & \glueavg{67.51} \\

        PiSSA
        & \gluescore{82.78}{1.58}
        & \gluescore{86.89}{1.44}
        & \gluescore{91.27}{0.16}
        & \gluescore{66.40}{0.79}
        & \gluescore{54.11}{0.57}
        & \gluescore{25.33}{0.68}
        & \glueavg{67.80} \\

        DoRA
        & \gluescore{80.51}{1.98}
        & \gluescore{88.78}{0.86}
        & \gluescore{91.08}{1.40}
        & \gluescore{67.00}{1.61}
        & \gluescore{55.45}{0.80}
        & \gluescore{23.75}{1.27}
        & \glueavg{67.76} \\

        rsLoRA
        & \gluescore{81.34}{1.23}
        & \gluescore{86.15}{1.09}
        & \gluescore{90.23}{0.78}
        & \gluescore{67.24}{1.24}
        & \gluescore{54.21}{0.68}
        & \gluescore{24.72}{0.95}
        & \glueavg{67.32} \\

        Fira
        & \gluescore{80.64}{1.47}
        & \gluescore{86.69}{1.16}
        & \gluescore{90.39}{0.96}
        & \gluescore{67.64}{1.12}
        & \gluescore{54.18}{0.77}
        & \gluescore{24.57}{1.25}
        & \glueavg{67.35} \\

        LoRA-GA
        & \gluescore{82.05}{1.05}
        & \gluescore{89.36}{0.98}
        & \gluescore{91.59}{0.28}
        & \gluescore{68.56}{1.08}
        & \gluescore{55.78}{0.54}
        & \gluescore{23.19}{0.89}
        & \glueavg{68.42} \\

        LoRA-Pro
        & \gluescore{83.26}{1.68}
        & \gluescore{89.24}{1.21}
        & \gluescore{91.95}{1.06}
        & \gluescore{67.83}{0.86}
        & \gluescore{55.42}{0.77}
        & \gluescore{25.45}{0.94}
        & \glueavg{68.86} \\

        \rowcolor{gdloraBlue}
        \textbf{GDLoRA}
        & \gluescore{\underline{84.81}}{1.24}
        & \gluescore{\textbf{90.69}}{0.97}
        & \gluescore{\underline{92.52}}{0.59}
        & \gluescore{\underline{69.61}}{0.95}
        & \gluescore{\underline{56.03}}{0.61}
        & \gluescore{\textbf{26.38}}{0.88}
        & \glueavg{\underline{70.01}} \\

        \bottomrule
    \end{tabular}%
    }
\end{table}

\subsection{Mathematical Reasoning}
\label{sec:results_math}

Table~\ref{tab:math_results} shows that GDLoRA achieves an average
accuracy of 70.01, outperforming standard LoRA at rank 8 by
2.50 percentage points.
It also exceeds LoRA-Pro, the strongest non-FFT baseline,
by 1.15 points and Fira by 2.66 points.
GDLoRA obtains the highest non-FFT score on all six tasks,
showing consistent improvements over both adapter-based
baselines and Fira in this setting.
The gains over standard LoRA are particularly pronounced on
AQuA (+3.94 points), MultiArith (+3.41), and AddSub (+3.29).
Relative to FFT, GDLoRA reduces the average gap from
3.12 points for standard LoRA to 0.62 points, closing
approximately 80.1\% of the original LoRA--FFT gap.
It exceeds FFT on MultiArith and AQuA, while ranking second
on the remaining four tasks.
These results demonstrate broad improvements across the
evaluated mathematical reasoning tasks and substantially
narrow the performance gap to FFT.

\subsection{Commonsense Reasoning}
\label{sec:results_commonsense}

As reported in Table~\ref{tab:commonsense_results}, GDLoRA
achieves an average accuracy of 85.15, surpassing standard
LoRA by 1.75 percentage points and the strongest non-FFT
baseline, LoRA-Pro, by 1.16 points.
It also outperforms Fira by 2.26 points on average and
achieves higher scores on all eight tasks.
GDLoRA obtains the highest non-FFT score on six tasks,
with particularly large gains over standard LoRA on
OpenBookQA (+4.00 points), Social IQa (+3.07), and
ARC-Challenge (+2.41).
Compared with FFT's average accuracy of 85.70, GDLoRA
reduces the remaining difference to 0.55 points and achieves
higher scores on ARC-Challenge, Social IQa, and ARC-Easy.
LoRA-Pro and LoRA-GA retain small advantages on PIQA and
BoolQ, respectively, but GDLoRA delivers the strongest
aggregate performance among the non-FFT methods.
These results extend the reasoning improvements observed
on Llama-3-8B to Gemma-7B.

\subsection{Image Classification}
\label{sec:results_vision}

Table~\ref{tab:vision_results} shows that GDLoRA achieves
average accuracies of 86.79 and 89.20 on ViT-Base and
ViT-Large, respectively, obtaining the highest average
scores among all evaluated methods on both backbones.
These results exceed standard LoRA at rank 8 by 2.07 and
2.46 percentage points.
Compared with Fira, GDLoRA improves average accuracy by
1.64 and 1.87 points and achieves higher scores in all
16 model--dataset combinations.
It also surpasses PiSSA, the strongest non-FFT baseline
on ViT-Large, by 1.40 points.
The largest gains over standard LoRA occur on FGVC-Aircraft,
reaching 5.10 and 8.07 points.
StanfordCars also improves substantially, by 3.72 and
7.23 points, with both fine-grained recognition tasks
showing larger gains on ViT-Large.
The improvements on these two datasets exceed the
corresponding average gains for both backbones, highlighting
fine-grained recognition as a strong setting for GDLoRA
within the evaluated image classification suite.
GDLoRA's average accuracies are also comparable to FFT's
86.43 and 89.08, extending its competitive performance
to visual adaptation.

\begin{table}[t]
    \centering
    \caption{
        Results on eight commonsense reasoning tasks with Gemma-7B.
        All LoRA-based methods use rank 8.
        The best and second-best scores among all methods, including
        FFT, are shown in \textbf{bold} and \underline{underlined},
        respectively.
    }
    \label{tab:commonsense_results}

    \footnotesize
    \setlength{\tabcolsep}{3pt}
    \renewcommand{\arraystretch}{1.15}

    \resizebox{1.0\linewidth}{!}{%
    \begin{tabular}{lccccccccc}
        \toprule
        Method & OBQA & ARC-c & WinoGrande & PIQA & SIQA & ARC-e & BoolQ & HellaSwag & Avg. \\
        \midrule

        FFT
        & \gluescore{\textbf{86.77}}{1.46}
        & \gluescore{\underline{84.93}}{1.66}
        & \gluescore{\textbf{88.05}}{1.25}
        & \gluescore{\textbf{89.11}}{1.07}
        & \gluescore{\underline{80.65}}{1.43}
        & \gluescore{\underline{93.06}}{1.28}
        & \gluescore{\textbf{70.12}}{0.24}
        & \gluescore{\textbf{92.89}}{0.84}
        & \glueavg{\textbf{85.70}} \\
        \midrule

        LoRA
        & \gluescore{81.62}{0.82}
        & \gluescore{82.74}{0.89}
        & \gluescore{85.08}{0.88}
        & \gluescore{87.54}{0.65}
        & \gluescore{77.79}{1.14}
        & \gluescore{92.67}{1.01}
        & \gluescore{68.84}{0.25}
        & \gluescore{90.95}{0.65}
        & \glueavg{83.40} \\

        PiSSA
        & \gluescore{81.24}{0.57}
        & \gluescore{81.06}{1.06}
        & \gluescore{81.69}{0.73}
        & \gluescore{86.24}{0.48}
        & \gluescore{75.74}{1.28}
        & \gluescore{91.12}{0.84}
        & \gluescore{67.43}{0.31}
        & \gluescore{89.13}{0.79}
        & \glueavg{81.71} \\

        DoRA
        & \gluescore{82.53}{0.94}
        & \gluescore{81.95}{0.68}
        & \gluescore{83.58}{1.02}
        & \gluescore{86.82}{0.72}
        & \gluescore{76.93}{0.96}
        & \gluescore{91.16}{1.12}
        & \gluescore{68.45}{0.22}
        & \gluescore{83.56}{0.58}
        & \glueavg{81.87} \\

        rsLoRA
        & \gluescore{82.86}{0.71}
        & \gluescore{81.83}{1.14}
        & \gluescore{84.50}{0.79}
        & \gluescore{87.63}{0.54}
        & \gluescore{77.12}{1.21}
        & \gluescore{92.13}{0.93}
        & \gluescore{68.09}{0.37}
        & \gluescore{92.23}{0.83}
        & \glueavg{83.30} \\

        Fira
        & \gluescore{83.12}{1.03}
        & \gluescore{81.66}{0.74}
        & \gluescore{84.16}{0.91}
        & \gluescore{87.92}{0.60}
        & \gluescore{75.84}{1.35}
        & \gluescore{91.26}{0.76}
        & \gluescore{68.71}{0.29}
        & \gluescore{90.42}{0.52}
        & \glueavg{82.89} \\

        LoRA-GA
        & \gluescore{82.68}{0.63}
        & \gluescore{82.72}{0.97}
        & \gluescore{85.95}{1.08}
        & \gluescore{87.19}{0.69}
        & \gluescore{78.84}{0.88}
        & \gluescore{91.69}{1.15}
        & \gluescore{\underline{68.93}}{0.21}
        & \gluescore{91.34}{0.74}
        & \glueavg{83.67} \\

        LoRA-Pro
        & \gluescore{83.66}{0.89}
        & \gluescore{82.89}{0.59}
        & \gluescore{85.75}{0.95}
        & \gluescore{\underline{88.65}}{0.77}
        & \gluescore{78.69}{1.18}
        & \gluescore{91.86}{0.81}
        & \gluescore{68.56}{0.34}
        & \gluescore{91.89}{0.61}
        & \glueavg{83.99} \\

        \rowcolor{gdloraBlue}
        \textbf{GDLoRA}
        & \gluescore{\underline{85.62}}{0.76}
        & \gluescore{\textbf{85.15}}{1.09}
        & \gluescore{\underline{86.27}}{0.84}
        & \gluescore{88.36}{0.57}
        & \gluescore{\textbf{80.86}}{1.26}
        & \gluescore{\textbf{93.77}}{0.98}
        & \gluescore{68.89}{0.27}
        & \gluescore{\underline{92.28}}{0.71}
        & \glueavg{\underline{85.15}} \\

        \bottomrule
    \end{tabular}%
    }
\end{table}


\begin{table}[t]
    \centering
    \caption{
        Image classification results with ViT-Base and ViT-Large.
        For each model, the best and second-best scores among all
        methods are shown in \textbf{bold} and
        \underline{underlined}, respectively.
    }
    \label{tab:vision_results}

    \footnotesize
    \setlength{\tabcolsep}{3pt}
    \renewcommand{\arraystretch}{1.15}

    \resizebox{1.0\linewidth}{!}{%
    \begin{tabular}{clccccccccc}
        \toprule
        Model & Method & OxfordPets & StanfordCars & CIFAR-10 & DTD & EuroSAT & FGVC & RESISC45 & CIFAR-100 & Avg. \\
        \midrule

        & FFT
        & \gluescore{\underline{96.13}}{0.21}
        & \gluescore{\underline{76.81}}{0.94}
        & \gluescore{\textbf{98.96}}{0.04}
        & \gluescore{\underline{77.75}}{0.88}
        & \gluescore{\underline{98.92}}{0.19}
        & \gluescore{\textbf{54.64}}{0.97}
        & \gluescore{\textbf{96.12}}{0.34}
        & \gluescore{\textbf{92.11}}{0.18}
        & \glueavg{\underline{86.43}} \\
        \cmidrule(l){2-11}

        & LoRA
        & \gluescore{95.92}{0.37}
        & \gluescore{74.93}{0.58}
        & \gluescore{97.26}{0.06}
        & \gluescore{75.03}{0.54}
        & \gluescore{98.57}{0.12}
        & \gluescore{49.17}{0.71}
        & \gluescore{95.43}{0.29}
        & \gluescore{91.44}{0.22}
        & \glueavg{84.72} \\

        & PiSSA
        & \gluescore{95.46}{0.18}
        & \gluescore{74.24}{0.83}
        & \gluescore{97.92}{0.03}
        & \gluescore{74.82}{0.97}
        & \gluescore{98.58}{0.05}
        & \gluescore{51.07}{0.52}
        & \gluescore{95.12}{0.17}
        & \gluescore{90.14}{0.13}
        & \glueavg{84.67} \\

        & DoRA
        & \gluescore{95.62}{0.29}
        & \gluescore{74.69}{0.41}
        & \gluescore{97.64}{0.07}
        & \gluescore{75.77}{0.63}
        & \gluescore{98.76}{0.11}
        & \gluescore{49.81}{0.86}
        & \gluescore{95.59}{0.23}
        & \gluescore{91.32}{0.20}
        & \glueavg{84.90} \\

        & rsLoRA
        & \gluescore{95.51}{0.33}
        & \gluescore{74.12}{0.69}
        & \gluescore{97.66}{0.05}
        & \gluescore{75.87}{0.45}
        & \gluescore{98.49}{0.08}
        & \gluescore{50.78}{0.61}
        & \gluescore{95.48}{0.12}
        & \gluescore{91.08}{0.16}
        & \glueavg{84.87} \\

        & Fira
        & \gluescore{95.06}{0.16}
        & \gluescore{75.28}{1.07}
        & \gluescore{98.54}{0.02}
        & \gluescore{76.19}{0.76}
        & \gluescore{98.87}{0.06}
        & \gluescore{51.22}{0.74}
        & \gluescore{95.08}{0.31}
        & \gluescore{90.95}{0.24}
        & \glueavg{85.15} \\

        & LoRA-GA
        & \gluescore{95.74}{0.42}
        & \gluescore{74.71}{0.35}
        & \gluescore{98.08}{0.07}
        & \gluescore{76.45}{0.59}
        & \gluescore{98.48}{0.13}
        & \gluescore{50.21}{0.49}
        & \gluescore{95.33}{0.20}
        & \gluescore{91.31}{0.19}
        & \glueavg{85.04} \\

        & LoRA-Pro
        & \gluescore{95.92}{0.24}
        & \gluescore{74.23}{0.92}
        & \gluescore{98.18}{0.04}
        & \gluescore{77.66}{1.14}
        & \gluescore{98.33}{0.10}
        & \gluescore{50.81}{1.08}
        & \gluescore{95.17}{0.27}
        & \gluescore{90.52}{0.15}
        & \glueavg{85.10} \\

        \rowcolor{gdloraBlue}
        \cellcolor{white}
        \multirow{-9}{*}{%
            \rotatebox[origin=c]{90}{ViT-Base}%
        }
        & \textbf{GDLoRA}
        & \gluescore{\textbf{97.01}}{0.15}
        & \gluescore{\textbf{78.65}}{0.47}
        & \gluescore{\underline{98.86}}{0.03}
        & \gluescore{\textbf{78.78}}{0.69}
        & \gluescore{\textbf{98.96}}{0.07}
        & \gluescore{\underline{54.27}}{0.65}
        & \gluescore{\underline{95.73}}{0.26}
        & \gluescore{\underline{92.02}}{0.21}
        & \glueavg{\textbf{86.79}} \\

        \midrule

        & FFT
        & \gluescore{\textbf{97.44}}{0.41}
        & \gluescore{\underline{85.87}}{0.84}
        & \gluescore{99.15}{0.04}
        & \gluescore{\textbf{81.17}}{0.72}
        & \gluescore{\underline{99.04}}{0.08}
        & \gluescore{\underline{60.17}}{1.12}
        & \gluescore{\textbf{96.21}}{0.09}
        & \gluescore{\underline{93.55}}{0.17}
        & \glueavg{\underline{89.08}} \\
        \cmidrule(l){2-11}

        & LoRA
        & \gluescore{95.92}{0.19}
        & \gluescore{78.91}{0.58}
        & \gluescore{99.16}{0.03}
        & \gluescore{79.12}{0.44}
        & \gluescore{98.35}{0.05}
        & \gluescore{53.59}{0.63}
        & \gluescore{95.56}{0.13}
        & \gluescore{93.34}{0.23}
        & \glueavg{86.74} \\

        & PiSSA
        & \gluescore{96.24}{0.33}
        & \gluescore{82.21}{1.29}
        & \gluescore{99.19}{0.13}
        & \gluescore{80.43}{0.91}
        & \gluescore{98.78}{0.07}
        & \gluescore{56.62}{1.27}
        & \gluescore{95.73}{0.08}
        & \gluescore{93.17}{0.14}
        & \glueavg{87.80} \\

        & DoRA
        & \gluescore{95.91}{0.27}
        & \gluescore{78.28}{0.66}
        & \gluescore{99.11}{0.07}
        & \gluescore{78.88}{0.35}
        & \gluescore{98.48}{0.09}
        & \gluescore{53.92}{0.52}
        & \gluescore{95.76}{0.12}
        & \gluescore{93.24}{0.18}
        & \glueavg{86.70} \\

        & rsLoRA
        & \gluescore{96.25}{0.52}
        & \gluescore{81.66}{0.47}
        & \gluescore{99.15}{0.13}
        & \gluescore{80.05}{0.63}
        & \gluescore{98.65}{0.04}
        & \gluescore{55.42}{0.88}
        & \gluescore{\underline{96.08}}{0.07}
        & \gluescore{93.15}{0.21}
        & \glueavg{87.55} \\

        & Fira
        & \gluescore{94.89}{0.14}
        & \gluescore{82.21}{0.61}
        & \gluescore{\underline{99.23}}{0.05}
        & \gluescore{80.21}{0.49}
        & \gluescore{98.77}{0.08}
        & \gluescore{55.52}{0.41}
        & \gluescore{95.24}{0.14}
        & \gluescore{92.56}{0.09}
        & \glueavg{87.33} \\

        & LoRA-GA
        & \gluescore{96.47}{0.24}
        & \gluescore{78.77}{0.38}
        & \gluescore{99.22}{0.08}
        & \gluescore{79.04}{0.83}
        & \gluescore{98.48}{0.03}
        & \gluescore{54.79}{0.74}
        & \gluescore{95.49}{0.11}
        & \gluescore{93.24}{0.16}
        & \glueavg{86.94} \\

        & LoRA-Pro
        & \gluescore{96.65}{0.08}
        & \gluescore{78.55}{0.54}
        & \gluescore{99.12}{0.09}
        & \gluescore{79.31}{0.57}
        & \gluescore{98.41}{0.12}
        & \gluescore{55.77}{0.97}
        & \gluescore{95.08}{0.10}
        & \gluescore{\textbf{93.76}}{0.12}
        & \glueavg{87.08} \\

        \rowcolor{gdloraBlue}
        \cellcolor{white}
        \multirow{-9}{*}{%
            \rotatebox[origin=c]{90}{ViT-Large}%
        }
        & \textbf{GDLoRA}
        & \gluescore{\underline{97.27}}{0.36}
        & \gluescore{\textbf{86.14}}{0.31}
        & \gluescore{\textbf{99.28}}{0.03}
        & \gluescore{\underline{80.74}}{0.68}
        & \gluescore{\textbf{99.16}}{0.07}
        & \gluescore{\textbf{61.66}}{1.35}
        & \gluescore{95.92}{0.13}
        & \gluescore{93.41}{0.15}
        & \glueavg{\textbf{89.20}} \\

        \bottomrule
    \end{tabular}%
    }
\end{table}

\subsection{Ablation Studies and Analysis}
\label{sec:ablation_analysis}

\begin{figure*}[t]
    \centering
    \begin{subfigure}[t]{0.33\linewidth}
        \centering
        \includegraphics[width=1.0\linewidth]
            {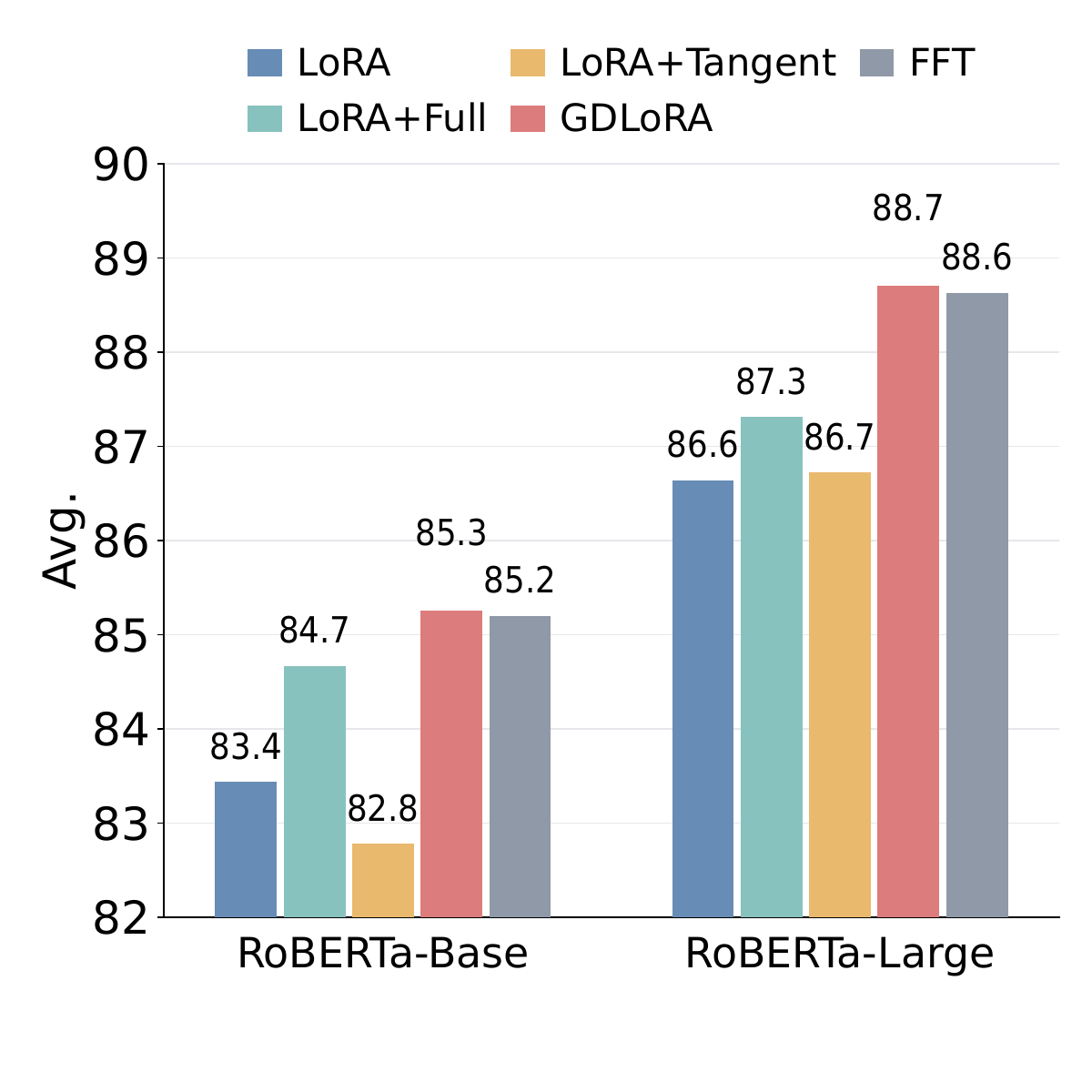}
        \caption{Gradient component ablation}
        \label{fig:gdlora_glue_ablation}
    \end{subfigure}\hfill%
    \begin{subfigure}[t]{0.33\linewidth}
        \centering
        \includegraphics[width=1.0\linewidth]
            {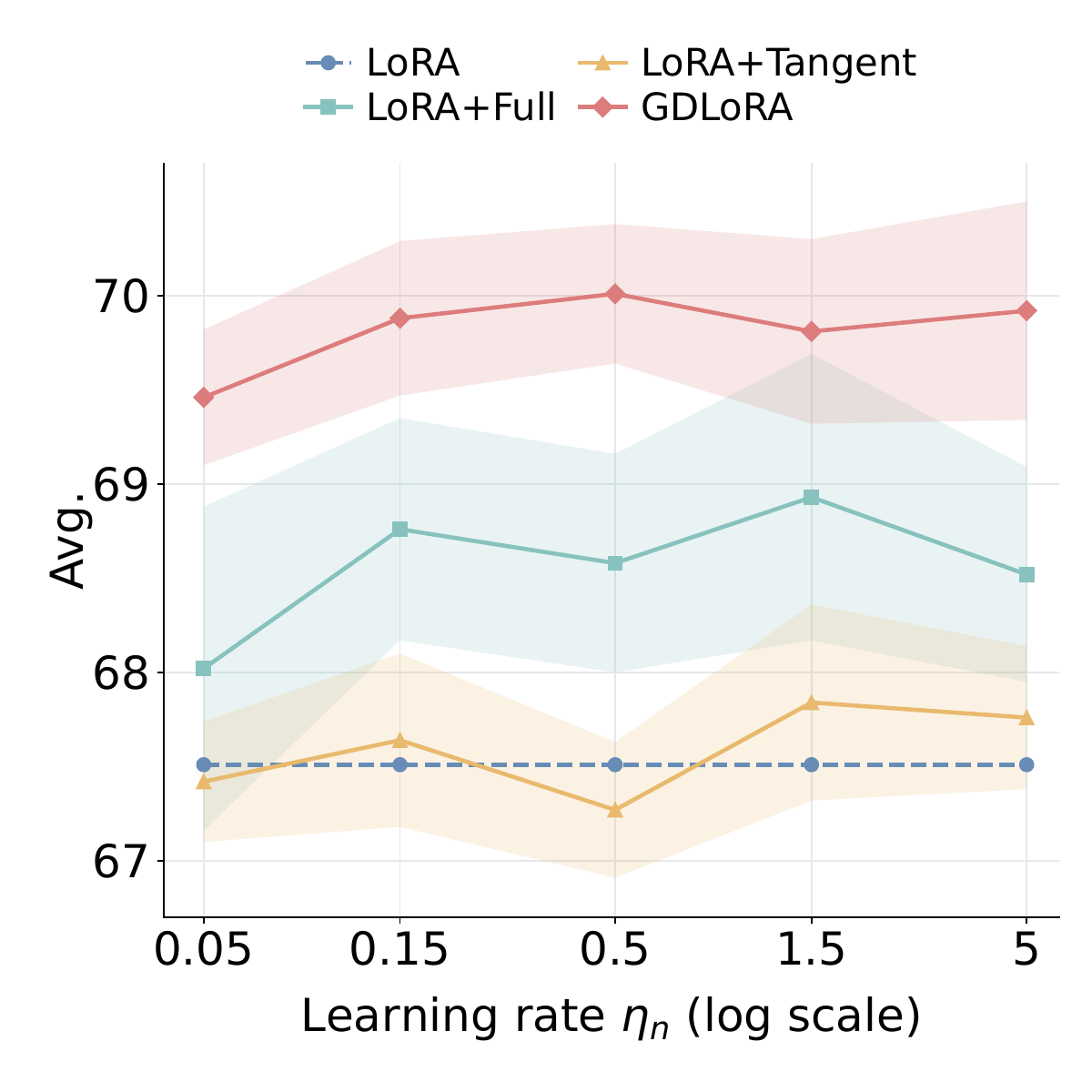}
        \caption{Learning-rate sensitivity}
        \label{fig:gdlora_lr_sensitivity}
    \end{subfigure}\hfill%
    \begin{subfigure}[t]{0.33\linewidth}
        \centering
        \includegraphics[width=1.0\linewidth]
            {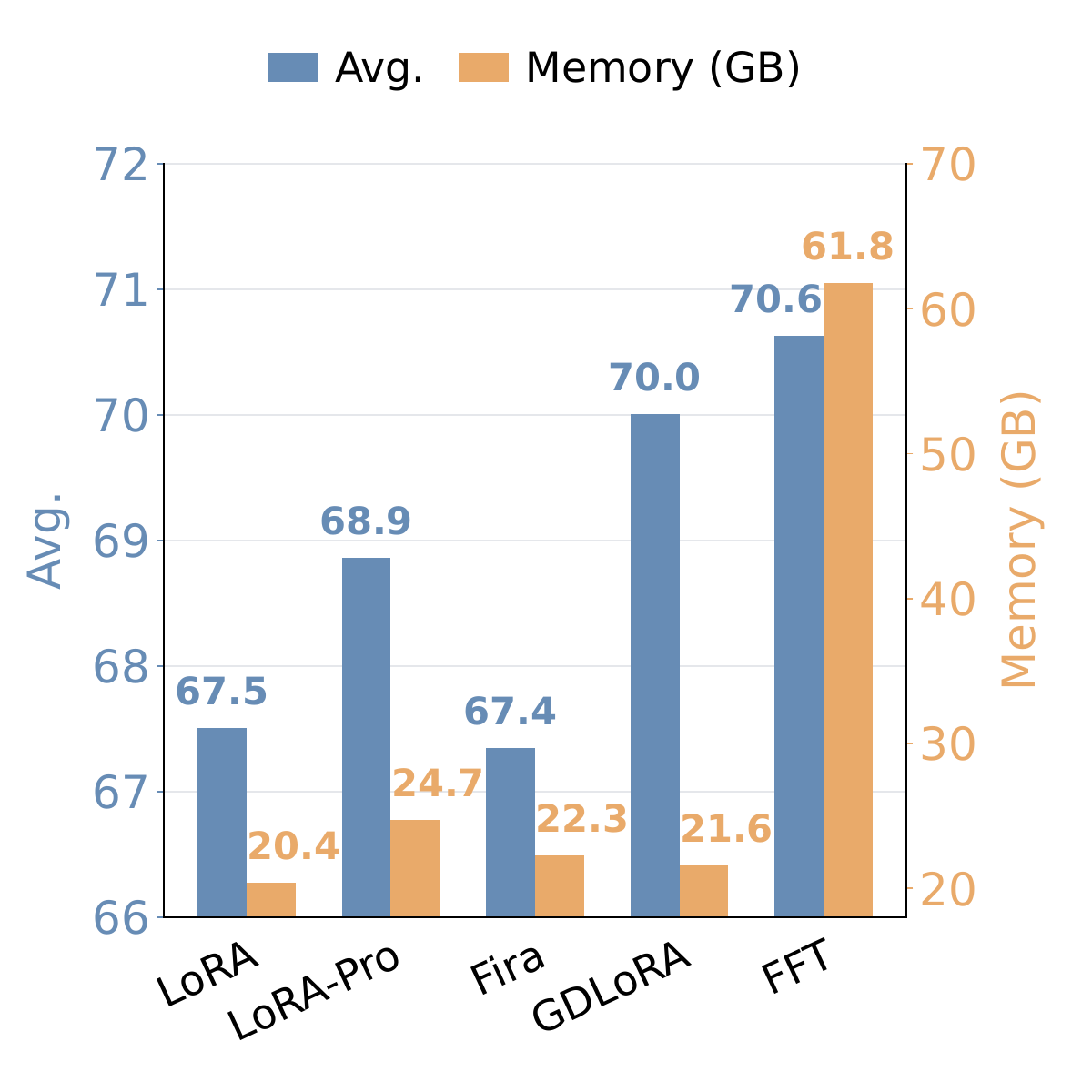}
        \caption{Accuracy and memory}
        \label{fig:gdlora_accuracy_memory}
    \end{subfigure}

    \caption{
        \textbf{Ablation studies and analysis of GDLoRA.}
        (a) Comparison of gradient components on six GLUE tasks
        with RoBERTa-Base and RoBERTa-Large.
        (b) Sensitivity to the base-update learning rate when
        fine-tuning Llama-3-8B on Math10K.
        Shaded bands indicate $\pm1$ standard deviation.
        (c) Average mathematical reasoning accuracy and memory usage
        with Llama-3-8B on Math10K.
    }
    \label{fig:gdlora_analysis}
\end{figure*}

\paragraph{Gradient component ablation.}
Figure~\ref{fig:gdlora_analysis}(a) evaluates the effect of the
gradient component used to update the base weights.
We compare GDLoRA with standard LoRA, LoRA+Full, and LoRA+Tangent,
where the latter two apply the full gradient and its tangent
component to the base weights, respectively.
GDLoRA improves the average GLUE score over LoRA by 1.82 and
2.07 points on RoBERTa-Base and RoBERTa-Large, respectively.
It also outperforms LoRA+Full by 0.59 and 1.40 points.
LoRA+Tangent reduces performance on RoBERTa-Base and yields
only a marginal improvement on RoBERTa-Large.
Tangent updates introduce no
directions outside the adapters' current first-order accessible
space. LoRA+Full also modifies updates within this space,
potentially interacting with the adapter optimizer.
GDLoRA isolates the normal component, supplying complementary
gradient information that the current adapters cannot represent
at first order.
The observed advantage is consistent with the benefit of this
complementary update, although overlap alone does not imply
that tangent updates are inherently harmful.
GDLoRA also achieves average scores numerically comparable
to FFT on both backbones.

\paragraph{Learning-rate sensitivity.}
Figure~\ref{fig:gdlora_analysis}(b) examines sensitivity to the
base-update learning rate on Llama-3-8B fine-tuned on Math10K.
We vary $\eta_n\in\{0.05,0.15,0.5,1.5,5\}$ for the methods
with additional base-weight updates, while standard LoRA
serves as a fixed baseline.
GDLoRA achieves mean accuracies between 69.46 and 70.01,
with a variation of 0.55 percentage points across the
tested range.
Its improvements over LoRA range from 1.95 to 2.50 percentage
points, with the highest accuracy obtained at $\eta_n=0.5$.
Moreover, its lowest mean score exceeds the best LoRA+Full mean score
of 68.93.
These results suggest that GDLoRA's gains persist across
the evaluated learning-rate range and do not depend on
a single narrowly selected setting.

\paragraph{Accuracy and memory.}
Figure~\ref{fig:gdlora_analysis}(c) compares GPU memory usage
during fine-tuning of Llama-3-8B on Math10K and the resulting
average accuracy across six mathematical reasoning tasks.
GDLoRA achieves an average accuracy of 70.01 with 21.6\,GB
of memory, improving over LoRA by 2.50 percentage points
while requiring only 1.2\,GB (5.9\%) additional memory.
Compared with LoRA-Pro and Fira, GDLoRA improves accuracy
by 1.15 and 2.66 percentage points while reducing memory
usage by 12.6\% and 3.1\%, respectively.
Relative to FFT, GDLoRA reduces memory usage from
61.79\,GB to 21.6\,GB, a reduction of 65.0\%,
while trailing by 0.62 percentage points in accuracy.
Thus, GDLoRA narrows the accuracy gap to FFT from
3.12 to 0.62 percentage points while incurring a modest
memory increase over standard LoRA.

\section{Conclusion}
\label{sec:conclusion}

We presented GDLoRA, which complements standard LoRA with gradient
information beyond the first-order directions accessible through its
current parameterization. By characterizing the LoRA-accessible gradient
space as the induced tangent space, GDLoRA isolates the normal component
of the full weight gradient and applies it directly to the base weights,
while retaining standard AdamW optimization for the adapters.
Under matched adapter and optimizer configurations, GDLoRA preserves
LoRA's optimizer-state memory budget.
Experiments on natural language understanding, mathematical reasoning,
commonsense reasoning, and image classification demonstrate consistent
improvements over LoRA and a reduced performance gap to FFT.
Ablation studies further support the effectiveness of normal-gradient
updates in complementing low-rank adaptation.

\bibliography{iclr2027_conference}

@inproceedings{devlin2019bert,
  title={Bert: Pre-training of deep bidirectional transformers for language understanding},
  author={Devlin, Jacob and Chang, Ming-Wei and Lee, Kenton and Toutanova, Kristina},
  booktitle={Proceedings of the 2019 conference of the North American chapter of the association for computational linguistics: human language technologies, volume 1 (long and short papers)},
  pages={4171--4186},
  year={2019}
}

@article{brown2020language,
  title={Language models are few-shot learners},
  author={Brown, Tom and Mann, Benjamin and Ryder, Nick and Subbiah, Melanie and Kaplan, Jared D and Dhariwal, Prafulla and Neelakantan, Arvind and Shyam, Pranav and Sastry, Girish and Askell, Amanda and others},
  journal={Advances in neural information processing systems},
  volume={33},
  pages={1877--1901},
  year={2020}
}

@article{chowdhery2023palm,
  title={Palm: Scaling language modeling with pathways},
  author={Chowdhery, Aakanksha and Narang, Sharan and Devlin, Jacob and Bosma, Maarten and Mishra, Gaurav and Roberts, Adam and Barham, Paul and Chung, Hyung Won and Sutton, Charles and Gehrmann, Sebastian and others},
  journal={Journal of machine learning research},
  volume={24},
  number={240},
  pages={1--113},
  year={2023}
}

@inproceedings{rajbhandari2020zero,
  title={Zero: Memory optimizations toward training trillion parameter models},
  author={Rajbhandari, Samyam and Rasley, Jeff and Ruwase, Olatunji and He, Yuxiong},
  booktitle={SC20: international conference for high performance computing, networking, storage and analysis},
  pages={1--16},
  year={2020},
  organization={IEEE}
}

@inproceedings{houlsby2019parameter,
  title={Parameter-efficient transfer learning for NLP},
  author={Houlsby, Neil and Giurgiu, Andrei and Jastrzebski, Stanislaw and Morrone, Bruna and De Laroussilhe, Quentin and Gesmundo, Andrea and Attariyan, Mona and Gelly, Sylvain},
  booktitle={International conference on machine learning},
  pages={2790--2799},
  year={2019},
  organization={PMLR}
}

@inproceedings{pfeiffer2020adapterhub,
  title={Adapterhub: A framework for adapting transformers},
  author={Pfeiffer, Jonas and R{\"u}ckl{\'e}, Andreas and Poth, Clifton and Kamath, Aishwarya and Vuli{\'c}, Ivan and Ruder, Sebastian and Cho, Kyunghyun and Gurevych, Iryna},
  booktitle={Proceedings of the 2020 conference on empirical methods in natural language processing: system demonstrations},
  pages={46--54},
  year={2020}
}

@inproceedings{li2021prefix,
  title={Prefix-tuning: Optimizing continuous prompts for generation},
  author={Li, Xiang Lisa and Liang, Percy},
  booktitle={Proceedings of the 59th annual meeting of the association for computational linguistics and the 11th international joint conference on natural language processing (volume 1: Long papers)},
  pages={4582--4597},
  year={2021}
}

@inproceedings{lester2021power,
  title={The power of scale for parameter-efficient prompt tuning},
  author={Lester, Brian and Al-Rfou, Rami and Constant, Noah},
  booktitle={Proceedings of the 2021 conference on empirical methods in natural language processing},
  pages={3045--3059},
  year={2021}
}

@article{hu2021lora,
  title={Lora: Low-rank adaptation of large language models},
  author={Hu, Edward J and Shen, Yelong and Wallis, Phillip and Allen-Zhu, Zeyuan and Li, Yuanzhi and Wang, Shean and Wang, Lu and Chen, Weizhu},
  journal={arXiv preprint arXiv:2106.09685},
  year={2021}
}

@article{wang2024lora,
  title={Lora-ga: Low-rank adaptation with gradient approximation},
  author={Wang, Shaowen and Yu, Linxi and Li, Jian},
  journal={Advances in Neural Information Processing Systems},
  volume={37},
  pages={54905--54931},
  year={2024}
}

@article{zhang2025lora,
  title={Lora-one: One-step full gradient could suffice for fine-tuning large language models, provably and efficiently},
  author={Zhang, Yuanhe and Liu, Fanghui and Chen, Yudong},
  journal={arXiv preprint arXiv:2502.01235},
  year={2025}
}

@inproceedings{wang2025lora,
  title={Lora-pro: Are low-rank adapters properly optimized?},
  author={Wang, Zhengbo and Liang, Jian and He, Ran and Wang, Zilei and Tan, Tieniu},
  booktitle={International Conference on Learning Representations},
  volume={2025},
  pages={93787--93808},
  year={2025}
}

@article{meng2024pissa,
  title={Pissa: Principal singular values and singular vectors adaptation of large language models},
  author={Meng, Fanxu and Wang, Zhaohui and Zhang, Muhan},
  journal={Advances in Neural Information Processing Systems},
  volume={37},
  pages={121038--121072},
  year={2024}
}

@article{hayou2024lora+,
  title={Lora+: Efficient low rank adaptation of large models},
  author={Hayou, Soufiane and Ghosh, Nikhil and Yu, Bin},
  journal={arXiv preprint arXiv:2402.12354},
  year={2024}
}

@article{kalajdzievski2023rank,
  title={A rank stabilization scaling factor for fine-tuning with lora},
  author={Kalajdzievski, Damjan},
  journal={arXiv preprint arXiv:2312.03732},
  year={2023}
}

@article{zhang2023adalora,
  title={Adalora: Adaptive budget allocation for parameter-efficient fine-tuning},
  author={Zhang, Qingru and Chen, Minshuo and Bukharin, Alexander and Karampatziakis, Nikos and He, Pengcheng and Cheng, Yu and Chen, Weizhu and Zhao, Tuo},
  journal={arXiv preprint arXiv:2303.10512},
  year={2023}
}

@article{liu2024dora,
  title={Dora: Weight-decomposed low-rank adaptation},
  author={Liu, Shih-Yang and Wang, Chien-Yi and Yin, Hongxu and Molchanov, Pavlo and Wang, Yu-Chiang Frank and Cheng, Kwang-Ting and Chen, Min-Hung},
  journal={arXiv preprint arXiv:2402.09353},
  year={2024}
}

@inproceedings{Wang2018GLUEAM,
  title={GLUE: A Multi-Task Benchmark and Analysis Platform for Natural Language Understanding},
  author={Alex Wang and Amanpreet Singh and Julian Michael and Felix Hill and Omer Levy and Samuel R. Bowman},
  booktitle={BlackboxNLP@EMNLP},
  year={2018},
  url={https://api.semanticscholar.org/CorpusID:5034059}
}

@article{Liu2019RoBERTaAR,
  title={RoBERTa: A Robustly Optimized BERT Pretraining Approach},
  author={Yinhan Liu and Myle Ott and Naman Goyal and Jingfei Du and Mandar Joshi and Danqi Chen and Omer Levy and Mike Lewis and Luke Zettlemoyer and Veselin Stoyanov},
  journal={ArXiv},
  year={2019},
  volume={abs/1907.11692},
  url={https://api.semanticscholar.org/CorpusID:198953378}
}

@inproceedings{Dubey2024TheL3,
  title={The Llama 3 Herd of Models},
  author={Dubey, Abhimanyu and others},
  year={2024},
  url={https://api.semanticscholar.org/CorpusID:271571434}
}

@article{Mesnard2024GemmaOM,
  title={Gemma: Open Models Based on Gemini Research and Technology},
  author={Gemma Team Thomas Mesnard and others},
  journal={ArXiv},
  year={2024},
  volume={abs/2403.08295},
  url={https://api.semanticscholar.org/CorpusID:268379206}
}

@article{Hu2023LLMAdaptersAA,
  title={LLM-Adapters: An Adapter Family for Parameter-Efficient Fine-Tuning of Large Language Models},
  author={Zhiqiang Hu and Yihuai Lan and Lei Wang and Wanyu Xu and Ee-Peng Lim and Roy Ka-Wei Lee and Lidong Bing and Soujanya Poria},
  journal={ArXiv},
  year={2023},
  volume={abs/2304.01933},
  url={https://api.semanticscholar.org/CorpusID:257921386}
}

@article{Dosovitskiy2020AnII,
  title={An Image is Worth 16x16 Words: Transformers for Image Recognition at Scale},
  author={Alexey Dosovitskiy and Lucas Beyer and Alexander Kolesnikov and Dirk Weissenborn and Xiaohua Zhai and Thomas Unterthiner and Mostafa Dehghani and Matthias Minderer and Georg Heigold and Sylvain Gelly and Jakob Uszkoreit and Neil Houlsby},
  journal={ArXiv},
  year={2020},
  volume={abs/2010.11929},
  url={https://api.semanticscholar.org/CorpusID:225039882}
}

@inproceedings{krause20133d,
  title={3d object representations for fine-grained categorization},
  author={Krause, Jonathan and Stark, Michael and Deng, Jia and Fei-Fei, Li},
  booktitle={Proceedings of the IEEE international conference on computer vision workshops},
  pages={554--561},
  year={2013}
}

@inproceedings{cimpoi2014describing,
  title={Describing textures in the wild},
  author={Cimpoi, Mircea and Maji, Subhransu and Kokkinos, Iasonas and Mohamed, Sammy and Vedaldi, Andrea},
  booktitle={Proceedings of the IEEE conference on computer vision and pattern recognition},
  pages={3606--3613},
  year={2014}
}

@article{helber2019eurosat,
  title={Eurosat: A novel dataset and deep learning benchmark for land use and land cover classification},
  author={Helber, Patrick and Bischke, Benjamin and Dengel, Andreas and Borth, Damian},
  journal={IEEE Journal of Selected Topics in Applied Earth Observations and Remote Sensing},
  volume={12},
  number={7},
  pages={2217--2226},
  year={2019},
  publisher={IEEE}
}

@article{cheng2017remote,
  title={Remote sensing image scene classification: Benchmark and state of the art},
  author={Cheng, Gong and Han, Junwei and Lu, Xiaoqiang},
  journal={Proceedings of the IEEE},
  volume={105},
  number={10},
  pages={1865--1883},
  year={2017},
  publisher={IEEE}
}

@inproceedings{parkhi2012cats,
  title={Cats and dogs},
  author={Parkhi, Omkar M and Vedaldi, Andrea and Zisserman, Andrew and Jawahar, CV},
  booktitle={2012 IEEE conference on computer vision and pattern recognition},
  pages={3498--3505},
  year={2012},
  organization={IEEE}
}

@article{krizhevsky2009learning,
  title={Learning multiple layers of features from tiny images},
  author={Krizhevsky, Alex and Hinton, Geoffrey and others},
  year={2009},
  publisher={Toronto, ON, Canada}
}

@article{maji2013fine,
  title={Fine-grained visual classification of aircraft},
  author={Maji, Subhransu and Rahtu, Esa and Kannala, Juho and Blaschko, Matthew and Vedaldi, Andrea},
  journal={arXiv preprint arXiv:1306.5151},
  year={2013}
}

@article{Chen2024FiraCW,
  title={Fira: Can We Achieve Full-rank Training of LLMs Under Low-rank Constraint?},
  author={Xi Chen and Kaituo Feng and Chang-Sheng Li and Xunhao Lai and Xiangyu Yue and Ye Yuan and Guoren Wang},
  journal={ArXiv},
  year={2024},
  volume={abs/2410.01623},
  url={https://api.semanticscholar.org/CorpusID:273026172}
}

@article{Zhao2024GaLoreML,
  title={GaLore: Memory-Efficient LLM Training by Gradient Low-Rank Projection},
  author={Jiawei Zhao and Zhenyu (Allen) Zhang and Beidi Chen and Zhangyang Wang and Anima Anandkumar and Yuandong Tian},
  journal={ArXiv},
  year={2024},
  volume={abs/2403.03507},
  url={https://api.semanticscholar.org/CorpusID:268253596}
}

@article{Liu2022FewShotPF,
  title={Few-Shot Parameter-Efficient Fine-Tuning is Better and Cheaper than In-Context Learning},
  author={Haokun Liu and Derek Tam and Mohammed Muqeeth and Jay Mohta and Tenghao Huang and Mohit Bansal and Colin Raffel},
  journal={ArXiv},
  year={2022},
  volume={abs/2205.05638},
  url={https://api.semanticscholar.org/CorpusID:248693283}
}

@article{Gao2024ParameterEfficientFW,
  title={Parameter-Efficient Fine-Tuning with Discrete Fourier Transform},
  author={Ziqi Gao and Qi-Chao Wang and Aochuan Chen and Zijing Liu and Bingzhe Wu and Liang Chen and Jia Li},
  journal={ArXiv},
  year={2024},
  volume={abs/2405.03003},
  url={https://api.semanticscholar.org/CorpusID:269605083}
}

@inproceedings{Kang2024MiSSRT,
  title={MiSS: Revisiting the Trade-off in LoRA with an Efficient Shard-Sharing Structure},
  author={Jiale Kang and Qingyu Yin},
  year={2024},
  url={https://api.semanticscholar.org/CorpusID:272832329}
}
\bibliographystyle{iclr2027_conference}

\appendix

\section{Gradient Subspace Drift and Projection Refresh Costs}
\label{app:projection_refresh}

We investigate the trade-off between the temporal relevance of
gradient projection bases and the computational cost of refreshing
them. Our comparison uses native Fira \citep{Chen2024FiraCW}, which
directly updates the base weights using low-rank adaptive
optimization and scaled residual updates.
We examine gradient subspace variation across training steps
and network modules, and evaluate how the projection refresh
interval affects accuracy and training time.

\paragraph{Gradient subspace drift and projection refresh.}
SVD-based gradient projection methods, including GaLore
\citep{Zhao2024GaLoreML} and Fira, periodically construct projection
bases from full weight gradients and reuse them between refreshes.
Let $G_t$ denote the full weight gradient at step $t$, and let
$\tau(t)$ denote the most recent projection refresh step.
For a left-sided projection with orthonormal basis
$P_{\tau(t)}$, the projected gradient and unscaled residual are
\begin{equation}
R_t=P_{\tau(t)}^{T}G_t,
\qquad
S_t=(I-P_{\tau(t)}P_{\tau(t)}^{T})G_t.
\label{eq:refresh_projected_gradient}
\end{equation}
The projected gradient, residual, and low-dimensional optimizer
moments continue to update at every step.
However, the subspace receiving adaptive optimization remains
fixed until the next basis refresh.

To examine the temporal relevance of this subspace, we record
the full weight gradient of
\texttt{model.layers.15.self\_attn.v\_proj}
during Llama-3-8B fine-tuning on Math10K.
At each of 50 optimizer steps, we perform an SVD, retain the
leading $r=8$ singular directions, and compute pairwise
subspace similarities.
As shown in Figure~\ref{fig:projection_refresh}(a), similarities
are relatively higher among early steps and substantially
lower for many pairs involving later steps, indicating drift
in the leading gradient subspace along the recorded trajectory.
Figure~\ref{fig:gradient_subspace_across_layers} provides
additional results for six selected attention projections
spanning layers 0, 15, and 31 and collectively covering query,
key, and value projections, showing that subspace variation
occurs across the selected modules and network depths.
Consequently, reusing an earlier projection basis can leave
the adaptive branch operating in a subspace that no longer
adequately represents the current leading gradient directions.
Although Fira still incorporates components outside this basis,
its residual branch does not maintain separate Adam first-
and second-moment estimates for each residual entry.
Instead, it applies norm-based scaling derived from the
projected branch.
Thus, delayed refreshes can make the allocation of
low-dimensional optimizer states less aligned with the
evolving leading gradient subspace.
The variation observed within the 50-step window also indicates
that refreshing every 50 steps need not eliminate subspace
mismatch between refreshes.
These observations motivate an explicit evaluation of the
accuracy and runtime effects of changing the projection
refresh interval.

\begin{figure*}[t]
    \centering
    \captionsetup[subfigure]{
        labelformat=parens,
        labelsep=space,
        font=small
    }

    \begin{subfigure}[t]{0.49\linewidth}
        \centering
        \vspace{0pt}
        \includegraphics[
            width=\linewidth,
            height=0.82\linewidth,
            keepaspectratio
        ]{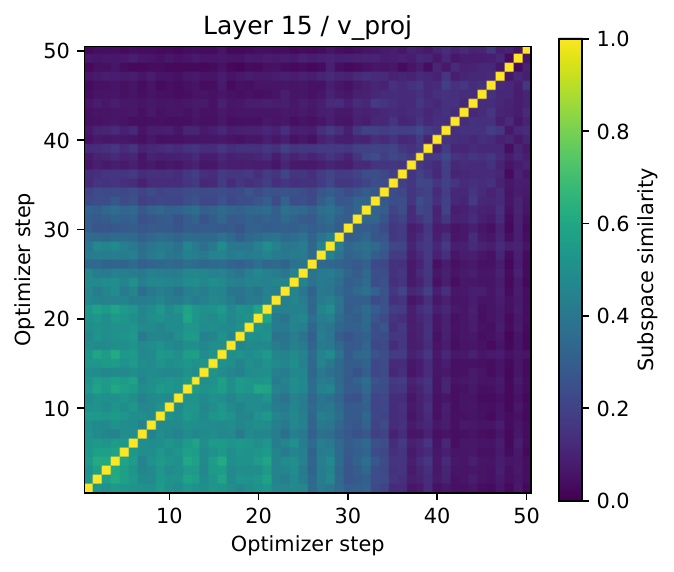}
        \caption{Gradient subspace drift}
        \label{fig:gradient_subspace_drift}
    \end{subfigure}
    \hfill
    \begin{subfigure}[t]{0.49\linewidth}
        \centering
        \vspace{0pt}
        \includegraphics[
            width=\linewidth,
            height=0.82\linewidth,
            keepaspectratio
        ]{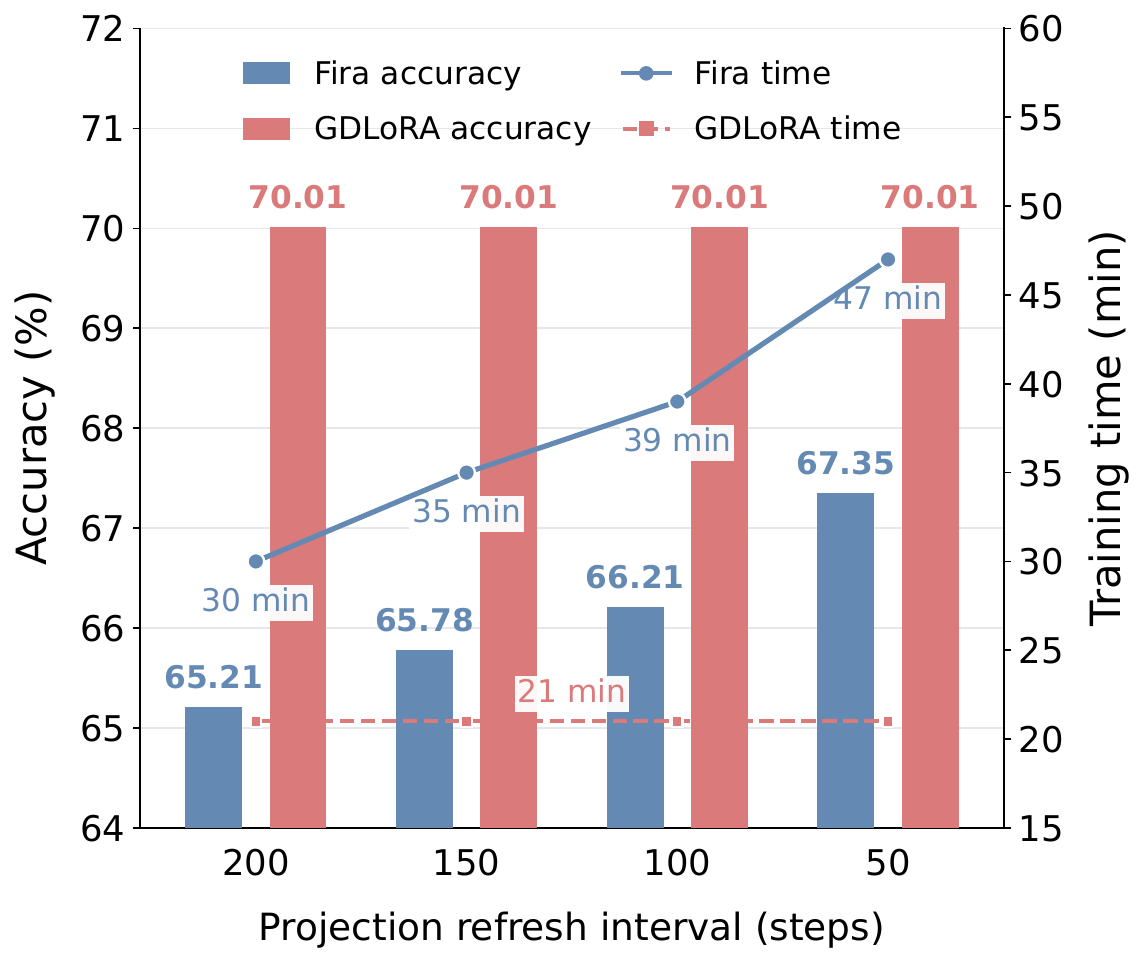}
        \caption{Accuracy and training time}
        \label{fig:fira_refresh_cost}
    \end{subfigure}

    \caption{
        \textbf{Gradient subspace drift and projection refresh costs.}
        \textbf{(a)} Pairwise similarity between the leading
        rank-8 gradient subspaces of
        \texttt{model.layers.15.self\_attn.v\_proj}
        over 50 optimizer steps.
        \textbf{(b)} Accuracy and training time of native Fira
        with projection refresh intervals of 200, 150, 100,
        and 50 steps on Math10K with Llama-3-8B.
        GDLoRA is shown as a fixed reference; its repeated
        values correspond to the same configuration.
    }
    \label{fig:projection_refresh}
\end{figure*}

\paragraph{More frequent refreshes incur substantial training costs.}
Figure~\ref{fig:projection_refresh}(b) evaluates native Fira
with refresh intervals $T\in\{200,150,100,50\}$.
Its corresponding accuracies are 65.21\%, 65.78\%, 66.21\%,
and 67.35\%, with training times of 30, 35, 39, and 47 minutes,
respectively.
Reducing the interval from 200 to 50 steps improves accuracy
by 2.14 percentage points, but increases training time by
56.7\%.
Thus, more frequent refreshes improve performance in this
experiment while imposing a substantial runtime penalty.

GDLoRA achieves 70.01\% accuracy in 21 minutes.
Even at $T=50$, the highest-accuracy Fira configuration in this
sweep remains 2.66 percentage points below GDLoRA and requires
$2.24\times$ its training time.
More frequent refreshes therefore do not close the observed
performance gap, despite their increased computational cost.
Such a runtime penalty can substantially increase the compute
requirements of longer training workloads and may limit
practicality under constrained training budgets.
The repeated GDLoRA values are a fixed reference, rather than
separate experiments with different refresh intervals.

These observations motivate placing adaptive optimization in the LoRA
factor coordinates rather than in a periodically refreshed gradient-SVD
subspace. GDLoRA maintains AdamW moments for the trainable factors $A$
and $B$, while deriving the normal component from the current adapter
space at each step. This design avoids repeated full-gradient SVD and
supplements adapter optimization with a direct normal-gradient update.

\paragraph{GDLoRA recomputes the normal component from current factors.}
GDLoRA avoids periodically estimating a gradient-SVD subspace
by constructing its projection bases directly from the current
LoRA factors:
\begin{equation}
U_t=\operatorname{orth}(B_t),
\qquad
V_t=\operatorname{orth}(A_t^{T}),
\label{eq:refresh_gdlora_bases}
\end{equation}
where $\operatorname{orth}(\cdot)$ returns an orthonormal basis
for the corresponding column space in the exact formulation.
Appendix~\ref{app:numerical_bases} specifies the numerical
thresholding procedure and its limitations for rank-deficient factors.
The normal gradient is then
\begin{equation}
G_{\perp,t}
=
(I-U_tU_t^{T})G_t(I-V_tV_t^{T}).
\label{eq:refresh_gdlora_normal}
\end{equation}
Both the gradient and projection bases are computed from the
same pre-update parameters.
With exact column-space bases, each normal update is orthogonal to the
first-order weight-space directions accessible through the
current adapters.

Basis construction operates on thin factors rather than
the full gradient matrix.
For full-rank factors, thin QR factorizations require
$O((d_o+d_i)r^2)$ operations.
Meanwhile, the adapter optimizer states remain in the
parameter coordinates of $A$ and $B$ and follow standard
AdamW updates.
Updating the projection bases therefore does not require
re-expressing the adapter moments in a new SVD basis.
The normal branch directly updates the base weights without
maintaining additional momentum or variance buffers for them.

GDLoRA thus recomputes its projection bases from the current
LoRA factors at every step without repeatedly computing an SVD
of the full weight gradient.
The cross-layer diagnostics demonstrate variation in leading
gradient subspaces across multiple attention projections,
while the refresh-interval experiment establishes the
accuracy--runtime trade-off in the evaluated Fira configuration.

\begin{figure*}[t]
    \centering

    \includegraphics[
        viewport=0 0 1082.160034 307.354401,
        clip,
        width=\linewidth
    ]{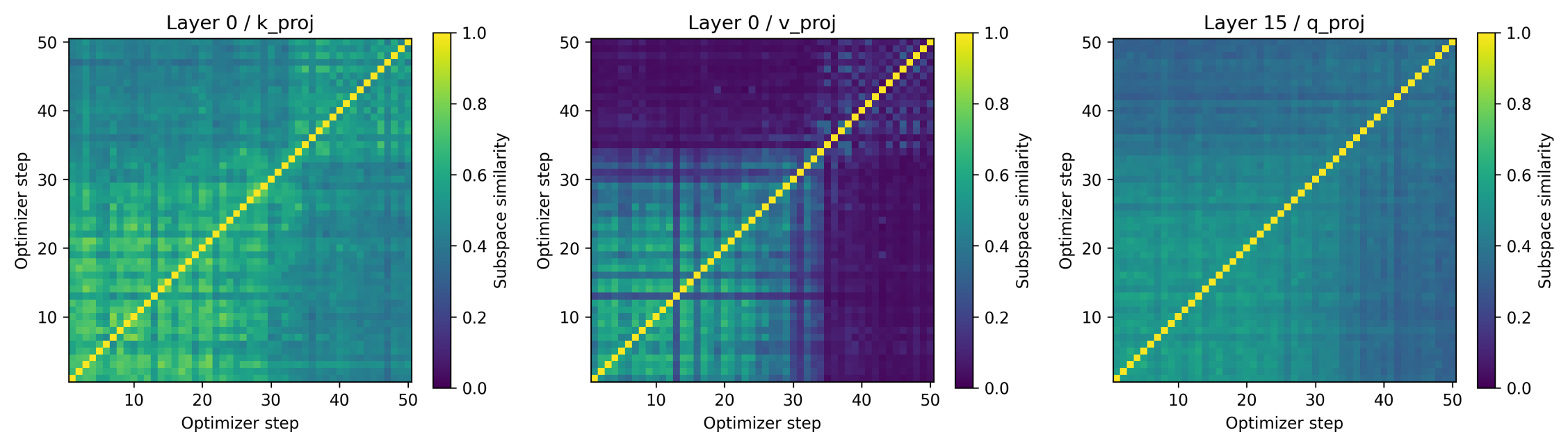}

    \par\vspace{0.6em}

    \includegraphics[
        viewport=1082.160034 0 2164.320068 307.354401,
        clip,
        width=\linewidth
    ]{selected_six_modules_row.pdf}

    \caption{
        \textbf{Gradient subspace variation across network depths
        and attention projections.}
        Each heatmap shows pairwise similarity between leading
        rank-8 gradient subspaces over 50 optimizer steps during
        Llama-3-8B fine-tuning on Math10K.
        From left to right, the top row shows the layer-0 key,
        layer-0 value, and layer-15 query projections;
        the bottom row shows the layer-15 value,
        layer-31 query, and layer-31 value projections.
        Layer numbers follow the model's module indices.
        The layer-15 value projection is also shown in
        Figure~\ref{fig:projection_refresh}(a).
        All panels use the same color scale.
    }
    \label{fig:gradient_subspace_across_layers}
\end{figure*}

\section{Computational Comparison with Fira}
\label{app:gdlora_computational_cost}

We compare the leading FLOPs of GDLoRA and native
Fira~\citep{Chen2024FiraCW} for a linear layer with weight
dimensions $d_o\times d_i$, rank $r$, and $N$ input positions.
A multiplication and an addition count as two FLOPs.
Shared base-layer forward and input-gradient computations,
as well as lower-order elementwise operations, are excluded.

Both methods compute the full weight gradient, costing
$2Nd_od_i$ FLOPs.
Their main difference in basis construction is the matrix
being decomposed.
GDLoRA orthonormalizes the thin factors $B$ and $A^T$.
Reduced Householder QR, including explicit basis formation,
has leading cost
\begin{equation}
C_{\mathrm{QR}}\approx4(d_o+d_i)r^2.
\end{equation}
Fira instead obtains its basis through SVD of the full
gradient, with dense-decomposition cost
\begin{equation}
C_{\mathrm{SVD}}
=O\bigl(d_od_i\min(d_o,d_i)\bigr).
\end{equation}
Refreshing every $T$ steps amortizes this cost to
$C_{\mathrm{SVD}}/T$ per step.

\begin{table}[t]
    \centering
    \caption{
        Leading FLOPs per layer and optimization step.
        Fira's SVD cost is amortized over its refresh
        interval $T$.
        GDLoRA recomputes both bases every step.
        Projection counts assume rank-$r$ bases and
        efficient matrix-product ordering.
    }
    \label{tab:gdlora_fira_flops}
    \small
    \setlength{\tabcolsep}{6pt}
    \renewcommand{\arraystretch}{1.15}
    \resizebox{0.85\linewidth}{!}{%
    \begin{tabular}{lcc}
        \toprule
        Operation & Fira & GDLoRA \\
        \midrule
        Full-weight gradient
        & $2Nd_od_i$
        & $2Nd_od_i$ \\
        Gradient projection and reconstruction
        & $6d_od_ir$
        & $8d_od_ir$ \\
        Basis construction
        & $C_{\mathrm{SVD}}/T$
        & $C_{\mathrm{QR}}$ \\
        Adapter forward and backward
        & $0$
        & $6Nr(d_o+d_i)$ \\
        \bottomrule
    \end{tabular}%
    }
\end{table}

GDLoRA evaluates its two-sided projection through thin-basis
multiplications without forming dense projection matrices.
Although this requires more projection FLOPs than Fira's
one-sided formulation, its basis construction is inexpensive:
for square layers of width $d$, thin-factor QR scales as
$O(dr^2)$, compared with $O(d^3)$ for a dense gradient SVD
per refresh.
For $r\ll d$, this makes recomputing the adapter-derived
bases at every step practical.
The corresponding end-to-end training-time comparison is
reported in Figure~\ref{fig:projection_refresh}(b) and
Table~\ref{tab:roberta_base_training_time}.

\begin{table}[!htbp]
    \centering
    \caption{Training time (hours) on six GLUE tasks with RoBERTa-Base.}
    \label{tab:roberta_base_training_time}
    \small
    \setlength{\tabcolsep}{6pt}
    \renewcommand{\arraystretch}{1.1}
    \begin{tabular}{lrrrrrr}
        \toprule
        Method & SST-2 & MRPC & CoLA & QNLI & RTE & STS-B \\
        \midrule
        FFT          & 10.16 & 1.64 & 5.78 & 19.56 & 2.21 & 3.64 \\
        LoRA         & 6.92 & 0.58 & 3.54 & 13.24 & 1.06 & 2.50 \\
        PiSSA        & 13.06 & 1.08 & 4.92 & 16.78 & 1.47 & 2.98 \\
        DoRA         & 9.84 & 0.65 & 4.03 & 13.78 & 1.19 & 2.52 \\
        rsLoRA       & 9.98 & 0.57 & 3.51 & 14.23 & 1.03 & 2.50 \\
        Fira         & 16.18 & 1.53 & 6.31 & 19.65 & 2.34 & 4.35 \\
        LoRA-GA      & 6.62 & 0.57 & 3.51 & 14.26 & 1.03 & 2.51 \\
        LoRA-Pro     & 7.07 & 0.61 & 5.62 & 15.77 & 1.11 & 4.06 \\
        \midrule
        GDLoRA       & 7.22 & 0.62 & 3.83 & 14.45 & 1.12 & 2.75 \\
        \bottomrule
    \end{tabular}
\end{table}

\section{Proof of the LoRA-Accessible Gradient Space Characterization}
\label{app:proof_gradient_space}

\begin{proof}[Proof of Theorem~\ref{thm:lora_gradient_space}]
We characterize the parameter-induced directions and their orthogonal
complement, then relate this complement to the kernel of the equivalent
gradient operator.

For fixed $s\neq 0$, the differential of $F_s(A,B)=sBA$ is
\begin{equation}
DF_s(A,B)[H_A,H_B]=s(BH_A+H_BA).
\end{equation}
Since a nonzero scalar does not change a linear subspace,
\begin{equation}
\begin{aligned}
\mathcal{T}_{\mathrm{LoRA}}
&=\{BH_A+H_BA\}\\
&=\left\{UK+LV^T
\;\middle|\;
K\in\mathbb{R}^{r_B\times d_i},\,
L\in\mathbb{R}^{d_o\times r_A}\right\}.
\end{aligned}
\end{equation}
Here $H_A\in\mathbb{R}^{r\times d_i}$ and
$H_B\in\mathbb{R}^{d_o\times r}$ vary freely.
For the second equality, write $B=UR_B$ and $A^T=VR_A$, where
$R_B$ and $R_A$ have full row rank. Every parameter-induced direction
then has the stated form. Conversely, the maps
$H_A\mapsto R_BH_A$ and $H_B\mapsto H_BR_A^T$ are surjective,
so every pair $(K,L)$ can be attained. When either factor has rank zero,
its basis is empty and the corresponding term vanishes.

For any $Z\in\mathbb{R}^{d_o\times d_i}$,
\begin{equation}
\langle Z,UK+LV^T\rangle_F
=\langle U^TZ,K\rangle_F+\langle ZV,L\rangle_F.
\end{equation}
Because $K$ and $L$ vary independently, this gives
\begin{equation}
\mathcal{T}_{\mathrm{LoRA}}^\perp
=\{Z\in\mathbb{R}^{d_o\times d_i}:U^TZ=0,\ ZV=0\}.
\label{eq:gdlora_appendix_normal_space}
\end{equation}

Next, consider $\Phi_s(H)=s^2(BB^TH+HA^TA)$. The symmetry of
$BB^T$ and $A^TA$ implies
\begin{equation}
\langle Z,\Phi_s(H)\rangle_F
=\langle\Phi_s(Z),H\rangle_F,
\end{equation}
so $\Phi_s$ is self-adjoint. Moreover,
\begin{equation}
\langle H,\Phi_s(H)\rangle_F
=s^2\left(\|B^TH\|_F^2+\|HA^T\|_F^2\right).
\end{equation}
Thus $\Phi_s(H)=0$ implies $B^TH=0$ and $HA^T=0$;
the converse follows directly from the definition of $\Phi_s$.
Using the column spaces spanned by $U$ and $V$, we obtain
\begin{equation}
\begin{aligned}
\ker(\Phi_s)
&=\{H:B^TH=0,\ HA^T=0\}\\
&=\{H:U^TH=0,\ HV=0\}
=\mathcal{T}_{\mathrm{LoRA}}^\perp.
\end{aligned}
\end{equation}

Finally, the range of a self-adjoint operator on a finite-dimensional
inner-product space is the orthogonal complement of its kernel. Hence
\begin{equation}
\begin{aligned}
\mathcal{G}_{\mathrm{LoRA}}
&=\operatorname{Range}(\Phi_s)=\ker(\Phi_s)^\perp\\
&=(\mathcal{T}_{\mathrm{LoRA}}^\perp)^\perp
=\mathcal{T}_{\mathrm{LoRA}}.
\end{aligned}
\end{equation}
Together with the basis representation above, this proves the theorem.
\end{proof}

\section{Orthogonal Projection onto the Normal Space}
\label{app:normal_projection}

Let $P_U=UU^{T}$ and $P_V=VV^{T}$. For the tangent space
$\mathcal{T}_{\mathrm{LoRA}}=\{UK+LV^{T}\}$,
the identity
$\langle Z,UK+LV^{T}\rangle_F
=\langle U^{T}Z,K\rangle_F+\langle ZV,L\rangle_F$
implies
\begin{equation}
\mathcal{T}_{\mathrm{LoRA}}^{\perp}
=\{Z\in\mathbb{R}^{d_o\times d_i}:U^{T}Z=0,\ ZV=0\}.
\label{eq:lora_normal_space}
\end{equation}
Define $G_{\perp}=(I-P_U)G(I-P_V)$. Since
$U^{T}(I-P_U)=0$ and $(I-P_V)V=0$, we have
$G_{\perp}\in\mathcal{T}_{\mathrm{LoRA}}^{\perp}$.
The residual satisfies
\begin{align}
G_{\parallel}=G-G_{\perp}
&=P_U G+GP_V-P_U GP_V\nonumber\\
&=U(U^{T}G)+\bigl((I-P_U)GV\bigr)V^{T}
\in\mathcal{T}_{\mathrm{LoRA}}.
\end{align}
Thus $G=G_{\parallel}+G_{\perp}$ is the unique orthogonal decomposition
into the tangent and normal spaces, establishing that
$G_{\perp}$ is the orthogonal projection of $G$ onto
$\mathcal{T}_{\mathrm{LoRA}}^{\perp}$.
By Theorem~\ref{thm:lora_gradient_space},
$\mathcal{G}_{\mathrm{LoRA}}=\mathcal{T}_{\mathrm{LoRA}}$;
hence $\langle G_{\perp},Z\rangle_F=0$ for every
$Z\in\mathcal{G}_{\mathrm{LoRA}}$.

\section{Finite-Step Analysis of the Two Update Branches}
\label{app:gdlora_local_analysis}

\newtheorem{finitetheorem}{Theorem}[section]
\newtheorem{finitecorollary}[finitetheorem]{Corollary}
\newtheorem{finiteproposition}[finitetheorem]{Proposition}

Section~\ref{sec:gradient_decomposition} extracts the normal gradient
relative to the current LoRA-accessible space. This appendix examines
whether finite adapter updates can themselves supply the missing normal
displacement. Although bilinear factor updates can produce such a
displacement, its reachable set is constrained by the factor-step size
and internal rank. We characterize these constraints and give a
sufficient condition under which direct normal compensation improves
matching to a full-gradient reference step. All statements use the
current, fixed space and exact column-space bases; the space may change
at subsequent iterations.

\subsection{Normal displacement of a finite adapter update}
\label{app:finite_normal_displacement}

Let
\begin{equation}
\Delta A_t=A_{t+1}-A_t,
\qquad
\Delta B_t=B_{t+1}-B_t
\end{equation}
denote the factor increments produced by the adapter optimizer.
The corresponding finite-step change in the adapter weight is
\begin{equation}
\begin{aligned}
D_t:=\Delta W_{\mathrm{adapter},t}
&=s(B_{t+1}A_{t+1}-B_tA_t)\\
&=s(B_t\Delta A_t+\Delta B_tA_t)
+s\Delta B_t\Delta A_t.
\end{aligned}
\label{eq:finite_adapter_displacement}
\end{equation}

The first-order term belongs to the tangent space induced by the
current factors:
\begin{equation}
D_{\mathrm{tan},t}
=s(B_t\Delta A_t+\Delta B_tA_t)
\in\mathcal{T}_{\mathrm{LoRA},t}.
\end{equation}
Since the normal gradient is computed using the same pre-update
factors, it satisfies
\begin{equation}
\langle G_{\perp,t},D_{\mathrm{tan},t}\rangle_F=0.
\end{equation}
This property holds for arbitrary factor increments, including
those generated by AdamW.

Let $P_{\perp,t}(Z)=(I-U_tU_t^T)Z(I-V_tV_t^T)$.
The actual normal displacement of the adapter update is therefore
\begin{equation}
Q_t:=P_{\perp,t}(D_t)
=s(I-U_tU_t^T)\Delta B_t\Delta A_t(I-V_tV_t^T).
\label{eq:finite_adapter_normal_displacement}
\end{equation}
Thus, a finite LoRA update can have a normal component, but that
component arises entirely from its bilinear remainder. In particular,
\begin{equation}
\langle G_{\perp,t},D_t\rangle_F
=\langle G_{\perp,t},Q_t\rangle_F
=s\langle G_{\perp,t},\Delta B_t\Delta A_t\rangle_F
\end{equation}
need not be zero. Including the base-weight update, the
total effective-weight displacement is
\begin{equation}
\begin{aligned}
W'_{t+1}-W'_t
={}&-\eta_{n,t}G_{\perp,t}
+D_{\mathrm{tan},t}\\
&+s\Delta B_t\Delta A_t.
\end{aligned}
\end{equation}

\subsection{Reachable normal displacements under a factor-step budget}
\label{app:finite_normal_reachability}

Fix $A\in\mathbb{R}^{r\times d_i}$,
$B\in\mathbb{R}^{d_o\times r}$, and $s\neq0$, and suppress the time
index. The factors need not have full rank. Write
$\mathcal T=\mathcal T_{\mathrm{LoRA}}$ and
$P_\perp(Z)=(I-UU^T)Z(I-VV^T)$, using the zero-rank convention
of Section~\ref{sec:gradient_decomposition}. For $\varepsilon\geq0$,
define
\begin{equation}
\mathcal R_\varepsilon
=\left\{P_\perp\!\left(s[(B+\Delta B)(A+\Delta A)-BA]\right)
\;\middle|\;
\|\Delta A\|_F^2+\|\Delta B\|_F^2\leq\varepsilon^2\right\}.
\label{eq:finite_normal_reachable_set}
\end{equation}
Here $\|\cdot\|_*$ denotes the nuclear norm, the sum of singular values.

\begin{finitetheorem}[Exact normal reachability]
\label{thm:finite_normal_reachability}
For the fixed factor coordinates above,
\begin{equation}
\boxed{\mathcal R_\varepsilon
=\left\{Z\in\mathcal T^\perp:
\operatorname{rank}(Z)\leq r,\quad
\|Z\|_*\leq\frac{|s|\varepsilon^2}{2}\right\}.}
\label{eq:finite_normal_reachability}
\end{equation}
\end{finitetheorem}

\begin{proof}
Let $\widetilde B=(I-UU^T)\Delta B$ and
$\widetilde A=\Delta A(I-VV^T)$.
Equation~\eqref{eq:finite_adapter_normal_displacement} gives
$Q=s\widetilde B\widetilde A\in\mathcal T^\perp$, with
$\operatorname{rank}(Q)\leq r$. The product inequality for the nuclear
norm, nonexpansiveness of orthogonal projections, and
$2ab\leq a^2+b^2$ yield
\begin{align}
\|Q\|_*
&\leq |s|\|\widetilde B\|_F\|\widetilde A\|_F
\leq |s|\|\Delta B\|_F\|\Delta A\|_F\nonumber\\
&\leq\frac{|s|}{2}
\bigl(\|\Delta B\|_F^2+\|\Delta A\|_F^2\bigr)
\leq\frac{|s|\varepsilon^2}{2}.
\label{eq:finite_normal_nuclear_bound}
\end{align}
This proves one inclusion.

Conversely, let a nonzero $Z\in\mathcal T^\perp$ have compact SVD
$Z=L\Sigma R^T$ and rank $k\leq r$. Because
$U^TZ=0$ and $ZV=0$, its nonzero singular vectors satisfy
$U^TL=0$ and $R^TV=0$. Let
$J_k=[I_k\;0]\in\mathbb{R}^{k\times r}$ and choose
\begin{equation}
\Delta B=\frac{1}{\sqrt{|s|}}L\Sigma^{1/2}J_k,
\qquad
\Delta A=\frac{\operatorname{sgn}(s)}{\sqrt{|s|}}
J_k^T\Sigma^{1/2}R^T.
\label{eq:finite_normal_constructive_factors}
\end{equation}
These increments satisfy $s\Delta B\Delta A=Z$ and are unchanged
by the corresponding left and right normal projectors. Hence the
normal projection of their finite adapter displacement is exactly $Z$.
Moreover,
\begin{equation}
\|\Delta A\|_F^2+\|\Delta B\|_F^2
=\frac{2\operatorname{tr}(\Sigma)}{|s|}
=\frac{2\|Z\|_*}{|s|}
\leq\varepsilon^2.
\end{equation}
The zero matrix is realized by zero increments, completing the proof.
\end{proof}

The theorem describes the normal projection, not the whole weight
displacement: the construction can also produce a tangent component.
It characterizes all admissible factor increments, not the subset
selected by a particular optimizer. The budget is a Euclidean step
budget in the fixed factor coordinates, rather than a memory or
compute budget; it need not be invariant to rescaling the factors.
Within these coordinates, finite-step normal displacement has rank
at most $r$ and nuclear norm at most quadratic in $\varepsilon$.

\begin{finitecorollary}[Optimal normal matching error]
\label{cor:finite_normal_matching}
Let $Y\in\mathcal T^\perp$ be a target normal displacement, with
singular values $\sigma_1\geq\cdots\geq\sigma_p\geq0$,
where $p=\min(d_o,d_i)$. Set $k=\min(r,p)$ and
$c=|s|\varepsilon^2/2$. Define
$a_i=(\sigma_i-\tau)_+$ for $i\leq k$, where
$(x)_+=\max\{x,0\}$. Take $\tau=0$ if
$\sum_{i=1}^k\sigma_i\leq c$; otherwise take
$\tau\in[0,\sigma_1]$ satisfying
$\sum_{i=1}^k(\sigma_i-\tau)_+=c$. Then
\begin{equation}
\min_{Q\in\mathcal R_\varepsilon}\|Y-Q\|_F^2
=\sum_{i=1}^k(\sigma_i-a_i)^2
+\sum_{i=k+1}^{p}\sigma_i^2
\geq\sum_{i=k+1}^{p}\sigma_i^2.
\label{eq:finite_normal_optimal_error}
\end{equation}
\end{finitecorollary}

\begin{proof}
By Theorem~\ref{thm:finite_normal_reachability}, feasible matrices
have rank at most $r$ and nuclear norm at most $c$. For fixed singular
values of $Q$, the trace inequality
$\langle Y,Q\rangle_F\leq\sum_i\sigma_i(Y)\sigma_i(Q)$
shows that aligning its singular vectors with those of $Y$ minimizes
the Frobenius error. This alignment is feasible in $\mathcal T^\perp$
because the nonzero singular vectors of $Y$ lie in the corresponding
normal subspaces. The remaining problem is
\begin{equation}
\min_{a_i\geq0,\,\sum_{i=1}^k a_i\leq c}
\sum_{i=1}^k(\sigma_i-a_i)^2
+\sum_{i=k+1}^p\sigma_i^2.
\end{equation}
The optimality conditions give the stated common-threshold solution.
It preserves the ordering of the singular values, and the resulting
matrix is reachable by Theorem~\ref{thm:finite_normal_reachability}.
The nonnegative first sum gives the spectral-tail lower bound.
\end{proof}

For $Y=-\eta_nG_\perp$, the spectral-tail term is
$\eta_n^2\sum_{i>r}\sigma_i(G_\perp)^2$, with an empty sum interpreted
as zero. Thus a nonzero tail cannot be matched by an adapter's normal
displacement even without a step-size constraint. The first term in
Eq.~\eqref{eq:finite_normal_optimal_error} additionally quantifies the
error imposed by the factor-step budget. These are best attainable
errors over admissible increments, not predictions of optimizer behavior.

\subsection{Effect of direct normal compensation}
\label{app:finite_normal_compensation}

Fix the same current state and the same candidate factor increments
for both updates. Let $D=s[(B+\Delta B)(A+\Delta A)-BA]$ and
$Q=P_\perp(D)$. Use the recovered full gradient $G$ with the
normalization of Section~\ref{sec:full_gradient_recovery}, and
$G_\perp=P_\perp(G)$. For $\eta_n>0$, take $-\eta_nG$ as the
reference weight step. The normal-component matching errors of the
adapter-only update $D$ and the compensated update $D-\eta_nG_\perp$
are, respectively,
\begin{align}
E_{\mathrm{LoRA}}
&=\|P_\perp(D+\eta_nG)\|_F
=\|Q+\eta_nG_\perp\|_F,\nonumber\\
E_{\mathrm{GD}}
&=\|P_\perp(D-\eta_nG_\perp+\eta_nG)\|_F
=\|Q\|_F.
\label{eq:finite_normal_compensation_errors}
\end{align}

\begin{finiteproposition}[Improvement from direct normal compensation]
\label{prop:finite_normal_compensation}
For the fixed candidate update above,
\begin{equation}
E_{\mathrm{LoRA}}^2-E_{\mathrm{GD}}^2
=\eta_n^2\|G_\perp\|_F^2
+2\eta_n\langle Q,G_\perp\rangle_F.
\label{eq:finite_normal_compensation_identity}
\end{equation}
If $\|\Delta A\|_F^2+\|\Delta B\|_F^2\leq\varepsilon^2$ and
$c=|s|\varepsilon^2/2$, then
\begin{equation}
E_{\mathrm{LoRA}}^2-E_{\mathrm{GD}}^2
\geq\eta_n\|G_\perp\|_F
\bigl(\eta_n\|G_\perp\|_F-2c\bigr).
\label{eq:finite_normal_compensation_bound}
\end{equation}
Consequently, $\eta_n\|G_\perp\|_F>2c$ guarantees
$E_{\mathrm{GD}}<E_{\mathrm{LoRA}}$.
\end{finiteproposition}

\begin{proof}
Expanding the squared norms in
Eq.~\eqref{eq:finite_normal_compensation_errors} gives the identity.
Theorem~\ref{thm:finite_normal_reachability} implies
$\|Q\|_F\leq\|Q\|_*\leq c$.
Cauchy--Schwarz then gives
$\langle Q,G_\perp\rangle_F\geq-c\|G_\perp\|_F$,
which proves the bound and the strict-improvement condition.
\end{proof}

The condition is sufficient, not necessary: the exact effect also
depends on the alignment of $Q$ with $G_\perp$. It identifies a regime
where the desired normal step exceeds what the bounded bilinear
remainder can compensate, so the direct branch improves normal
matching for every candidate factor update within that budget.
The reference is a current full-gradient step under the paper's
normalization, not an FFT AdamW update. This comparison holds at the
same state with identical candidate factor increments; it does not
order the losses or complete training trajectories of the two methods.
Together, the results explain both the finite-step normal displacement
already available to LoRA and the additional role of direct normal
compensation, while retaining the first-order interpretation in the
main text.

\begin{table*}[t]
    \centering
    \caption{
        Hyperparameters used for GDLoRA in the natural language
        understanding experiments on the GLUE benchmark.
    }
    \label{tab:glue_hyperparameters}
    \small
    \setlength{\tabcolsep}{5pt}
    \renewcommand{\arraystretch}{1.15}
    \resizebox{1.0\linewidth}{!}{%
    \begin{tabular}{clcccccc}
        \toprule
        Model & Hyperparameter
        & SST-2 & MRPC & CoLA & QNLI & RTE & STS-B \\
        \midrule

        \multicolumn{2}{l}{Optimizer}
        & \multicolumn{6}{c}{AdamW} \\
        \multicolumn{2}{l}{Warmup ratio}
        & \multicolumn{6}{c}{0.06} \\
        \multicolumn{2}{l}{Learning-rate schedule}
        & \multicolumn{6}{c}{Linear} \\

        \midrule

        \multirow{7}{*}{\rotatebox[origin=c]{90}{RoBERTa-Base}}
        & \# GPUs
        & \multicolumn{6}{c}{1} \\

        & Epochs
        & 60 & 30 & 80 & 25 & 160 & 80 \\

        & Learning rate (LoRA)
        & \multicolumn{6}{c}{1e-4} \\

        & Normal-gradient learning rate
        & \multicolumn{6}{c}{1} \\

        & LoRA rank
        & \multicolumn{6}{c}{8} \\

        & Maximum sequence length
        & \multicolumn{6}{c}{512} \\

        & Batch size per GPU
        & \multicolumn{6}{c}{32} \\

        \midrule

        \multirow{7}{*}{\rotatebox[origin=c]{90}{RoBERTa-Large}}
        & \# GPUs
        & \multicolumn{6}{c}{1} \\

        & Epochs
        & 20 & 40 & 40 & 20 & 40 & 40 \\

        & Learning rate (LoRA)
        & \multicolumn{6}{c}{1e-4} \\

        & Normal-gradient learning rate
        & \multicolumn{6}{c}{1} \\

        & LoRA rank
        & \multicolumn{6}{c}{8} \\

        & Maximum sequence length
        & \multicolumn{6}{c}{128} \\

        & Batch size per GPU
        & \multicolumn{6}{c}{32} \\

        \bottomrule
    \end{tabular}%
    }
\end{table*}

\section{More Experimental Results}
\label{app:more_experimental_results}

\subsection{Comparison with Additional PEFT Methods}
\label{app:additional_peft_comparison}

Table~\ref{tab:additional_peft_math} compares GDLoRA with standard LoRA
($r=8$), the established LoRA variant AdaLoRA~\citep{zhang2023adalora},
the additional PEFT methods IA$^3$\citep{Liu2022FewShotPF} and FourierFT\citep{Gao2024ParameterEfficientFW}, and the recent LoRA
variant MiSS\citep{Kang2024MiSSRT}, using Llama-3-8B on six mathematical reasoning tasks. GDLoRA achieves the highest average accuracy of 70.01\%,
exceeding LoRA by 2.50 percentage points and MiSS by 2.43 points.
It obtains the best scores on MultiArith, SingleEq, SVAMP, and GSM8K,
while IA$^3$ and FourierFT lead on AddSub and AQuA, respectively.

\begin{table}[!htbp]
    \centering
    \caption{Additional comparison on six mathematical reasoning tasks with
    Llama-3-8B. Scores are accuracy (\%); Avg. is the arithmetic mean across
    the six tasks.
    The best mean score in each column is shown in \textbf{bold}.}
    \label{tab:additional_peft_math}
    \small
    \setlength{\tabcolsep}{4pt}
    \renewcommand{\arraystretch}{1.1}
    \resizebox{\linewidth}{!}{%
    \begin{tabular}{lrrrrrrr}
        \toprule
        Method & AddSub & MultiArith & SingleEq & SVAMP & GSM8K & AQuA & Avg. \\
        \midrule
        LoRA ($r=8$)
        & 81.52 & 87.28
        & 91.54 & 66.93
        & 55.34 & 22.44 & 67.51 \\
        AdaLoRA & 80.57 & 84.22 & 89.47 & 64.80 & 52.12 & 25.44 & 66.10 \\
        IA$^3$ & \textbf{86.84} & 83.50 & 86.42 & 64.50 & 50.64 & 26.38 & 66.38 \\
        FourierFT & 85.06 & 88.67 & 88.98 & 64.30 & 45.40 & \textbf{27.95} & 66.73 \\
        MiSS & 84.67 & 86.93 & 88.71 & 64.62 & 55.21 & 25.35 & 67.58 \\
        \midrule
        \rowcolor{gdloraBlue}
        GDLoRA
        & 84.81 & \textbf{90.69}
        & \textbf{92.52} & \textbf{69.61}
        & \textbf{56.03} & 26.38 & \textbf{70.01} \\
        \bottomrule
    \end{tabular}%
    }
\end{table}

\subsection{Effect of Adapter Rank}
\label{app:rank_scaling}

Table~\ref{tab:gdlora_rank_scaling} compares LoRA and GDLoRA at ranks
$r\in\{8,16,32\}$ on Llama-3-8B. GDLoRA achieves average accuracies of
70.01\%, 71.96\%, and 72.81\%, exceeding LoRA at the same ranks by
2.50, 2.29, and 2.01 percentage points, respectively.
GDLoRA at rank 8 also outperforms LoRA at rank 16, while GDLoRA at
rank 16 outperforms LoRA at rank 32. Rank-32 GDLoRA achieves the highest
scores on five tasks; AQuA peaks at rank 16. These results show that
GDLoRA's advantage persists across the evaluated ranks, although
individual task scores do not always improve monotonically.

\begin{table}[!htbp]
    \centering
    \caption{LoRA and GDLoRA across adapter ranks on six mathematical
    reasoning tasks with Llama-3-8B. Scores are accuracy (\%); Avg. is the
    arithmetic mean across tasks. GDLoRA rows are shaded blue, and the
    best score in each column is shown in \textbf{bold}.}
    \label{tab:gdlora_rank_scaling}
    \small
    \setlength{\tabcolsep}{4pt}
    \renewcommand{\arraystretch}{1.1}
    \resizebox{\linewidth}{!}{%
    \begin{tabular}{lrrrrrrr}
        \toprule
        Method & AddSub & MultiArith & SingleEq & SVAMP & GSM8K & AQuA & Avg. \\
        \midrule
        LoRA ($r=8$) & 81.52 & 87.28 & 91.54 & 66.93 & 55.34 & 22.44 & 67.51 \\
        LoRA ($r=16$) & 83.56 & 91.34 & 92.43 & 68.81 & 56.97 & 24.89 & 69.67 \\
        LoRA ($r=32$) & 85.14 & 91.87 & 93.45 & 69.72 & 57.66 & 26.95 & 70.80 \\
        \midrule
        \rowcolor{gdloraBlue}
        GDLoRA ($r=8$) & 84.81 & 90.69 & 92.52 & 69.61 & 56.03 & 26.38 & 70.01 \\
        \rowcolor{gdloraBlue}
        GDLoRA ($r=16$) & 88.15 & 92.16 & 93.78 & 70.85 & 57.94 & \textbf{28.89} & 71.96 \\
        \rowcolor{gdloraBlue}
        GDLoRA ($r=32$) & \textbf{90.36} & \textbf{94.48} & \textbf{94.23} & \textbf{71.35} & \textbf{58.47} & 27.94 & \textbf{72.81} \\
        \bottomrule
    \end{tabular}%
    }
\end{table}

\section{Experimental Details}
\label{app:experimental_details}

\paragraph{Natural Language Understanding.}
For natural language understanding, we use RoBERTa-Base and
RoBERTa-Large~\citep{Liu2019RoBERTaAR} on six tasks from the GLUE
benchmark~\citep{Wang2018GLUEAM}: SST-2, MRPC, CoLA, QNLI, RTE, and STS-B.
We apply rank-8 LoRA adapters to the query and value projections
in each Transformer layer.

\paragraph{Mathematical and Commonsense Reasoning.}
For mathematical reasoning, we fine-tune Llama-3-8B~\citep{Dubey2024TheL3} on Math10K~\citep{Hu2023LLMAdaptersAA} and
evaluate it on six tasks: AddSub, MultiArith, SingleEq, SVAMP, GSM8K,
and AQuA.
For commonsense reasoning, we fine-tune Gemma-7B~\citep{Mesnard2024GemmaOM} on Commonsense170K~\citep{Hu2023LLMAdaptersAA} and
evaluate it on eight tasks: OpenBookQA, ARC-Challenge, WinoGrande,
PIQA, Social IQa, ARC-Easy, BoolQ, and HellaSwag.
The two reasoning training collections contain approximately 10K and
170K examples, respectively.
In both reasoning settings, we use rank-8 adapters for the query, key,
and value projections.

\paragraph{Image Classification.}
For image classification, we evaluate ViT-Base and
ViT-Large~\citep{Dosovitskiy2020AnII} on eight datasets: OxfordPets~\citep{parkhi2012cats},
StanfordCars~\citep{krause20133d}, CIFAR-10~\citep{krizhevsky2009learning}, DTD~\citep{cimpoi2014describing}, EuroSAT~\citep{helber2019eurosat}, FGVC-Aircraft~\citep{maji2013fine}, RESISC45~\citep{cheng2017remote}, and CIFAR-100~\citep{krizhevsky2009learning}.

\paragraph{Baseline Methods.}
We compare GDLoRA with full fine-tuning (FFT), standard
LoRA~\citep{hu2021lora}, Fira~\citep{Chen2024FiraCW}, and five LoRA variants:
PiSSA~\citep{meng2024pissa}, DoRA~\citep{liu2024dora},
rsLoRA~\citep{kalajdzievski2023rank},
LoRA-GA~\citep{wang2024lora}, and
LoRA-Pro~\citep{wang2025lora}.
FFT updates all model parameters during fine-tuning.
The LoRA variants seek to narrow the performance gap between LoRA
and FFT through improved initialization, weight parameterization,
scaling, or optimization.
During fine-tuning, their updates remain governed by their
respective adapter parameterizations.
Fira instead combines low-rank gradient projection with scaled residual
updates to enable full-rank weight adaptation.
GDLoRA supplements standard LoRA with normal-gradient updates
applied directly to the base weights, extending adaptation beyond
the first-order weight-space directions accessible through the
current LoRA parameterization.

\begin{table}[t]
    \centering
    \caption{
        Training and evaluation settings for mathematical reasoning
        with Llama-3-8B fine-tuned on Math10K.
    }
    \label{tab:math_hyperparameters}
    \small
    \renewcommand{\arraystretch}{1.08}
    \begin{tabular*}{0.85\linewidth}{
        @{\extracolsep{\fill}}lc@{}
    }
        \toprule
        Hyperparameter & Setting \\
        \midrule

        Optimizer & AdamW \\
        Learning rate (LoRA) & 1e-4 \\
        Normal-gradient learning rate & 0.5 \\
        Weight decay & 0.0 \\
        Warmup ratio & 0.06 \\
        Learning-rate schedule & Linear \\
        Epochs & 1 \\

        \midrule

        LoRA rank & 8 \\
        Target modules & Q, K, V \\
        Training precision & BF16 \\
        \# GPUs per run & 1 \\
        Micro-batch size per GPU & 2 \\
        Gradient accumulation steps & 8 \\
        Effective training batch size & 16 \\
        Maximum sequence length & 256 \\
        \bottomrule
    \end{tabular*}
\end{table}

\paragraph{Experimental Hyperparameters}

The hyperparameters used by GDLoRA for natural language understanding,
mathematical reasoning, commonsense reasoning, and image classification
are provided in Tables~\ref{tab:glue_hyperparameters},
\ref{tab:math_hyperparameters},
\ref{tab:commonsense_hyperparameters}, and
\ref{tab:vit_hyperparameters}, respectively.
Within each comparison, all LoRA-based methods use the same adapter rank
and target modules as GDLoRA. Apart from the main learning rate,
the other training and optimization settings are also held consistent.
For each language model--dataset setting (GLUE, Math10K, and
Commonsense170K), we independently search for the best main-optimizer
learning rate for every method, including FFT, Fira, GDLoRA, and all
LoRA variants, using the same candidate set
$\{1,3,5,7,9,10\}\times10^{-5}$.
This grid starts at $10^{-5}$ with increments of $2\times10^{-5}$,
and additionally includes the upper endpoint $10^{-4}$.
The selected rate is used for full-parameter optimization in FFT and
adapter optimization in LoRA-based methods.
The ``Learning rate (LoRA)'' entries in the language-task tables report
the selected rates for GDLoRA; its normal-gradient learning rate is
reported separately.
Image-classification experiments retain the shared settings in
Table~\ref{tab:vit_hyperparameters}.

For Fira, the SVD-based gradient projection bases are refreshed every
50 optimizer steps, except in the refresh-interval study in
Appendix~\ref{app:projection_refresh}, where the interval is varied explicitly.

All experiments were conducted on a server equipped with
8 NVIDIA A100 GPUs.
The GPU counts reported in the tables refer to individual training runs.
Each experiment was repeated with three different random seeds,
and we report the mean and standard deviation across the three runs.
\begin{table}[t]
    \centering
    \caption{
        Training and evaluation settings for commonsense reasoning
        with Gemma-7B fine-tuned on Commonsense170K.
    }
    \label{tab:commonsense_hyperparameters}
    \small
    \renewcommand{\arraystretch}{1.08}
    \begin{tabular*}{0.85\linewidth}{
        @{\extracolsep{\fill}}lc@{}
    }
        \toprule
        Hyperparameter & Setting \\
        \midrule

        Optimizer & AdamW \\
        Learning rate (LoRA) & 3e-5 \\
        Normal-gradient learning rate & 0.5 \\
        Weight decay & 0.0 \\
        Warmup ratio & 0.06 \\
        Learning-rate schedule & Linear \\
        Epochs & 1 \\

        \midrule

        LoRA rank & 8 \\
        Target modules & Q, K, V \\
        Training precision & BF16 \\
        \# GPUs per run & 1 \\
        Micro-batch size per GPU & 2 \\
        Gradient accumulation steps & 8 \\
        Effective training batch size & 16 \\
        Maximum sequence length & 256 \\

        \bottomrule
    \end{tabular*}
\end{table}

\begin{table*}[t]
    \centering
    \caption{
        Hyperparameters used for GDLoRA in the image classification
        experiments with ViT-Base and ViT-Large.
        The normal-gradient learning rate controls the direct
        updates to the base weights.
    }
    \label{tab:vit_hyperparameters}
    \small
    \setlength{\tabcolsep}{4pt}
    \renewcommand{\arraystretch}{1.15}
    \resizebox{1.0\linewidth}{!}{%
    \begin{tabular}{clcccccccc}
        \toprule
        Model & Hyperparameter
        & \shortstack{Oxford\\Pets}
        & \shortstack{Stanford\\Cars}
        & CIFAR-10 & DTD & EuroSAT
        & FGVC & RESISC45 & CIFAR-100 \\
        \midrule

        \multicolumn{2}{l}{Optimizer}
        & \multicolumn{8}{c}{AdamW} \\
        \multicolumn{2}{l}{Weight decay}
        & \multicolumn{8}{c}{0.01} \\
        \multicolumn{2}{l}{Learning-rate schedule}
        & \multicolumn{8}{c}{Linear} \\
        \multicolumn{2}{l}{Epochs}
        & \multicolumn{8}{c}{20} \\

        \midrule

        \multirow{6}{*}{\rotatebox[origin=c]{90}{ViT-Base}}
        & \# GPUs
        & \multicolumn{8}{c}{1} \\

        & Learning rate (head)
        & 5e-3 & 5e-2 & 5e-2 & 5e-2
        & 5e-2 & 5e-2 & 5e-2 & 1e-2 \\

        & Learning rate (LoRA)
        & 5e-3 & 1e-2 & 1e-2 & 1e-2
        & 1e-2 & 1e-2 & 1e-2 & 1e-2 \\

        & Normal-gradient learning rate
        & \multicolumn{8}{c}{1e-3} \\

        & LoRA rank
        & \multicolumn{8}{c}{8} \\

        & Batch size per GPU
        & \multicolumn{8}{c}{128} \\

        \midrule

        \multirow{6}{*}{\rotatebox[origin=c]{90}{ViT-Large}}
        & \# GPUs
        & \multicolumn{8}{c}{1} \\

        & Learning rate (head)
        & 2e-2 & 1e-2 & 5e-3 & 1e-2
        & 2e-2 & 1e-2 & 1e-2 & 5e-3 \\

        & Learning rate (LoRA)
        & 1e-2 & 1e-2 & 1e-2 & 1e-2
        & 1e-2 & 1e-2 & 1e-2 & 1e-2 \\

        & Normal-gradient learning rate
        & \multicolumn{8}{c}{5e-2} \\

        & LoRA rank
        & \multicolumn{8}{c}{8} \\

        & Batch size per GPU
        & \multicolumn{8}{c}{32} \\

        \bottomrule
    \end{tabular}%
    }
\end{table*}

\section{Magnitude Control and Local Descent of the Normal Update}
\label{app:normal_update_stability}

We analyze one normal-branch update with the adapters and all other
weights fixed. Let $N>0$ count the input positions,
and write the recovered gradient as
\begin{equation}
H=\nabla_{W_{\mathrm{base}}}\mathcal{L}
=\delta^{T}X=\sum_{i=1}^{N}h_i,
\qquad h_i=\delta_i x_i^{T},
\end{equation}
where $x_i$ and $\delta_i$ are column vectors for position $i$.
At the current parameters, define the fixed orthogonal projector
$P_{\perp}(Z)=(I-UU^{T})Z(I-VV^{T})$ and set
$G=H/N$, $G_{\perp}=P_{\perp}(G)$, and
$\Delta W=-\eta_nG_{\perp}$ for $\eta_n>0$.
This analysis, like Eq.~\eqref{eq:gdlora_normal_descent}, assumes
one common positive divisor, an exact orthogonal projector, and an
update applied without rounding. The fixed-count accumulation convention
is specified in Section~\ref{app:normal_accumulation}.

\begin{proposition}[Magnitude control and local descent]
\label{prop:normal_update_stability}
The normal update satisfies
\begin{equation}
\|\Delta W\|_F
\leq \frac{\eta_n}{N}\sum_{i=1}^{N}\|h_i\|_F.
\label{eq:normal_magnitude_bound}
\end{equation}
In particular, if $\|h_i\|_F\leq C$ for all positions, then
$\|\Delta W\|_F\leq\eta_n C$.
If $\mathcal{L}$ has a $\beta$-Lipschitz gradient with respect to
$W_{\mathrm{base}}$ on a neighborhood containing the update segment,
where $\beta>0$, then
\begin{equation}
\mathcal{L}(W_{\mathrm{base}}-\eta_nG_{\perp})
\leq \mathcal{L}(W_{\mathrm{base}})
-\eta_n\left(N-\frac{\beta\eta_n}{2}\right)\|G_{\perp}\|_F^2.
\label{eq:normal_smooth_descent}
\end{equation}
Consequently, for $G_{\perp}\neq0$ and $0<\eta_n<2N/\beta$,
the normal step strictly decreases the current loss.
\end{proposition}

\begin{proof}
An orthogonal projector is nonexpansive in the Frobenius norm.
The triangle inequality therefore gives
\begin{equation}
\|\Delta W\|_F
=\frac{\eta_n}{N}\|P_{\perp}(H)\|_F
\leq\frac{\eta_n}{N}\|H\|_F
\leq\frac{\eta_n}{N}\sum_{i=1}^{N}\|h_i\|_F.
\end{equation}
The uniform bound follows immediately. Self-adjointness and idempotence
of $P_{\perp}$ also imply
\begin{equation}
\langle H,G_{\perp}\rangle_F
=N\langle G,P_{\perp}(G)\rangle_F
=N\|G_{\perp}\|_F^2.
\end{equation}
Applying the smoothness inequality to $\Delta W=-\eta_nG_{\perp}$ yields
\begin{align}
\mathcal{L}(W_{\mathrm{base}}+\Delta W)
&\leq\mathcal{L}(W_{\mathrm{base}})
+\langle H,\Delta W\rangle_F
+\frac{\beta}{2}\|\Delta W\|_F^2\nonumber\\
&=\mathcal{L}(W_{\mathrm{base}})
-\eta_n\left(N-\frac{\beta\eta_n}{2}\right)\|G_{\perp}\|_F^2.
\end{align}
This proves the descent claim.
\end{proof}

With a uniform per-position bound $C$, the corresponding unnormalized
update has bound $\eta_n NC$. Dividing by $N$ removes this explicit
linear dependence from the magnitude bound, without requiring independent
token contributions. This common positive scaling preserves the normal direction
and its orthogonality to the current adapter tangent space.

These results support magnitude control of the normal-gradient branch.
Any loss reduction is already included in $\delta$; division by $N$ is an
additional step-size scaling, equivalent to using $\eta_n/N$ for the
unnormalized normal gradient when $N$ is fixed.
This equivalence does not establish invariance to sequence length.
The descent result concerns a single normal step with fixed adapters and
does not guarantee descent of the joint AdamW and base-weight update or
convergence of the full training procedure.

\subsection{Microbatch accumulation with a fixed input count}
\label{app:normal_accumulation}

In the reported experiments, the normalization count is held fixed
across microbatches within each configuration. Let an optimizer step
contain $M$ microbatches, each with input activations
$X_j\in\mathbb{R}^{N_\mu\times d_i}$, where $N_\mu=bL$ is determined
by the fixed micro-batch size $b$ and input sequence dimension $L$.
Let $\delta_j$ denote the backward signal,
including the loss reduction and microbatch weighting applied by the
trainer, but with any AMP loss scaling removed. Writing
$H_j=\delta_j^TX_j$ and $H=\sum_jH_j$, the normal branch uses
\begin{equation}
\widetilde{G}_t
=\sum_{j=1}^{M}\frac{H_j}{N_\mu}
=\frac{H}{N_\mu},
\qquad
G_{\perp,t}=\frac{1}{N_\mu}P_{\perp,t}(H).
\label{eq:normal_microbatch_accumulation}
\end{equation}
Thus the additional normalization applies one common positive scale
to the accumulated gradient, preserving its direction and the relative
weighting of the microbatches. If $\ell_j$ denotes the trainer-weighted
loss contribution, then $H=\nabla\mathcal{L}_{\mathrm{acc}}$ for
$\mathcal{L}_{\mathrm{acc}}=\sum_j\ell_j$; both branches therefore use
derivatives of the same accumulated objective under this convention.
The normal-gradient controller performs no further division by $M$;
the trainer's accumulation weighting is already contained in $\delta_j$.

The effective learning rate for the accumulated unnormalized normal
gradient is $\eta_n/N_\mu$. For a different fixed input count $N_\mu'$,
the same ideal update at fixed $H$ and projection bases is obtained by
setting $\eta_n'=\eta_nN_\mu'/N_\mu$.
This specifies the scaling convention when changing the microbatch
shape; the reported GDLoRA comparisons retain their stated fixed
microbatch and accumulation configurations.
$N_\mu$ counts rows of the collated input tensor, not unmasked target
tokens. With padding to the same fixed tensor shape, differences in
unpadded sequence lengths therefore do not change this divisor.

All microbatches use the same pre-update parameters. At the optimizer
boundary, we construct $U_t,V_t$ from $A_t,B_t$ and project the accumulated
gradient before the adapter optimizer step. We cache this residual,
update the adapters with AdamW, and then apply the cached normal update
to the base weights. If AMP skips the optimizer step, the base update is
also skipped; the accumulation and residual buffers are then cleared.
For both reasoning settings in Tables~\ref{tab:math_hyperparameters}
and~\ref{tab:commonsense_hyperparameters}, each run uses one GPU,
micro-batch size 2, and 8 accumulation steps, yielding an effective
batch size of 16. The maximum sequence length is 256, and the fixed
count $N_\mu$ is determined by the input tensor shape used in each run.

\subsection{Numerical bases and zero-rank factors}
\label{app:numerical_bases}

For each factor matrix $F\in\{B,A^T\}$, the implementation uses the
absolute tolerance $\epsilon_{\mathrm{orth}}=10^{-6}$.
It removes columns whose Euclidean norms do not exceed this tolerance.
If no columns remain, it returns an empty basis with shape
$\operatorname{rows}(F)\times0$. Otherwise, it computes an unpivoted
reduced QR factorization of the retained columns and keeps only those
columns of $Q$ whose corresponding diagonal entries of $R$ satisfy
$|R_{ii}|>\epsilon_{\mathrm{orth}}$.

An empty basis has a zero associated projector, and its projection
subtraction is skipped. In particular, $B_0=0$ gives an empty $U_0$ and
\begin{equation}
G_{\perp,0}=\widetilde{G}_0(I-V_0V_0^T).
\end{equation}
If both bases are empty, $G_{\perp,0}=\widetilde{G}_0$.
Thus, arbitrary columns returned by QR on an all-zero matrix are not
used as a nonzero projection basis.

This thresholded QR procedure explicitly handles zero initialization,
but it is not a general rank-revealing factorization: for nonzero
rank-deficient factors, unpivoted QR with diagonal thresholding need
not recover the complete column space. The exact tangent-space
statements assume bases of the true column spaces; the numerical
implementation projects using the retained bases described above.

\end{document}